\documentclass[oneside,11pt]{article}
\usepackage{jmlr2e}

\usepackage{graphicx}
\usepackage{subfigure}
\usepackage{comment}
\usepackage{graphicx}
\usepackage{amsmath}
\usepackage{gensymb}
\usepackage{hyperref}
\usepackage{xcolor}
\usepackage{bbold, wrapfig}

\usepackage[linecolor=blue!60!,backgroundcolor=blue!10!,textwidth=1.5cm,textsize=footnotesize]{todonotes}
\newtheorem{assumption}{Assumption}

\newcommand{\OH}[1]{{\color{purple} #1}}

\jmlrheading{23}{2022}{1-39}{10/21; Revised
3/22}{4/22}{21-1270}{Olympio Hacquard, Krishnakumar Balasubramanian, Gilles Blanchard, Clément Levrard, Wolfgang Polonik}
\ShortHeadings{Topologically penalized regression on manifolds}{Hacquard, Balasubramanian, Blanchard, Levrard, Polonik}

\begin{document}

\title{Topologically penalized regression on manifolds 
}

\author{\name Olympio Hacquard \email olympio.hacquard@universite-paris-saclay.fr \\
       \addr Laboratoire de Math\'ematiques d'Orsay\\
       Université Paris-Saclay, CNRS, Inria\\
       91400 Orsay France
      \AND 
      \name Krishnakumar Balasubramanian \email kbala@ucdavis.edu \\
       \addr Department of Statistics\\
       University of California\\
       Davis, CA 95616 USA
             \AND 
      \name Gilles Blanchard  \email gilles.blanchard@universite-paris-saclay.fr \\
       \addr Laboratoire de Math\'ematiques d'Orsay\\
        Université Paris-Saclay, CNRS, Inria\\
       91400 Orsay France
      \AND 
      \name Cl\'ement Levrard \email clement.levrard@lpsm.paris \\
       \addr Laboratoire de Probabilités et Statistiques Mathématiques,\\
       Universit\'e de Paris, \\
       75013 Paris France
      \AND 
      \name Wolfgang Polonik \email wpolonik@ucdavis.edu \\
       \addr Department of Statistics\\
       University of California\\
       Davis, CA 95616 USA 
}

\editor{Sayan Mukherjee}

\maketitle




\maketitle

\begin{abstract}

We study a regression problem on a compact manifold $\mathcal{M}$. In order to take advantage of the underlying geometry and topology of the data, the regression task is performed on the basis of the first several eigenfunctions of the Laplace-Beltrami operator of the manifold, that are regularized with topological penalties. The proposed penalties are based on the topology of the sub-level sets of either the eigenfunctions or the estimated function.
The overall approach is shown to yield promising and competitive performance on various applications to both synthetic and real data sets. We also provide theoretical guarantees on the regression function estimates, on both its prediction error and its smoothness (in a topological sense). Taken together, these results support the relevance of our approach in the case where the targeted function is ``topologically smooth".
\end{abstract}



\section{Introduction}
\label{intro}
Problems of regression on manifolds are of growing importance in statistical learning. Given a manifold $\mathcal{M}$, the specific goal is to retrieve a true regression function $f^\star:\mathcal{M} \to \mathbb{R}$ from data $X_i$ (for $i=1,\ldots n)$ that lie on the manifold $\mathcal{M}$ and noisy real-valued responses of the form $Y_i = f^\star(X_i) + \varepsilon_i$ where $\varepsilon_i$ are the additive noise. Such problems arise in many applications where the data samples $X_i$, although represented by very high dimensional spaces like sets of images and 3D volumes, often have an underlying low-dimensional structure and lie on a manifold. This is in particular the case in medical applications. For instance, \cite{guerrero2014manifold} and \cite{jie2015manifold} study regression problems on a set of images of brains. While the set of all images is of very large dimension (the number of pixels), the set of brain images turns out to have a comparatively very small intrinsic dimension. Although there are ways to recover the metric of the underlying unknown manifold~\citep{beg2005computing}, in this article we adopt an extrinsic approach. 

A standard approach for estimating $f^\star$ is based on expanding the function in a suitable basis to take advantage of the underlying manifold structure. To this end, we consider the Laplace-Beltrami operator~\citep{rosenberg1997laplacian}, which has been broadly studied, both for its theoretical properties \citep{zelditch2009local, zelditch2017eigenfunctions, shi2010gradient} and its great power of applicability in statistical data analysis \citep{hendriks1990nonparametric,levy2006laplace, coifman2006diffusion,  saito2008data, mahadevan2007proto, gobel2018, kocak2020spectral}. In our context, we chose the basis to be the set of eigenfunctions of the Laplace-Beltrami operator. Since it is impossible to have access to a closed form expression for the Laplace-Beltrami eigenfunctions for various manifolds in full generality, we replace them by the eigenvectors of the Laplacian matrix of a graph built on the data; see \cite{mohar1991laplacian} for a complete treatment. Using the eigenvectors of the graph Laplacian for diverse learning tasks is an idea that has its roots in the works of \cite{belkin2003laplacian} and has become extremely popular since. There is a plethora of literature on this topic, and we refer to \cite{wang2015trend} for a theoretical treatment, and \cite{belkin2006manifold} and \cite{chun2016eigenvector} for two out of many applications. The use of the graph Laplacian spectrum is backed-up by solid guarantees regarding its convergence towards the spectrum of the Laplace-Beltrami operator. We refer to \cite{von2008consistency} for general results adopting the point of view of spectral clustering, \cite{burago2013graph} for a more recent treatment, and \cite{trillos2020error} for the recent generalization of the latter to random data.

In order to efficiently estimate $f^\star$, we will use a penalization procedure; see \cite{giraud2014introduction} or \cite{massart2007concentration} for a complete treatment of these methods. Specifically, we will present two types of penalties that both leverage topological information. These penalties are based on persistent homology, a field that has its origins in algebraic topology and Morse theory \citep{milnor1963morse}. The use of persistent homology has become increasingly popular over the past decade, popularized, among others, by the books \cite{edelsbrunner2010computational} and \cite{boissonnat2018geometric}. It offers a new approach to data representations. Penalties based on persistence follow a heuristic similar to the one based on total variation (see, for instance, \citealp{rudin1992nonlinear, hutter2016optimal}) which works by reducing the oscillations of the estimated function in order to reconstruct a smooth function. While the heuristics for total variation penalties and persistence based penalties are similar, they still work quite differently, as discussed below.

Penalizing the persistence has been used recently in \cite{chen2019topological} for classification applications, and in \cite{bruel2019topology} in the context of Generative Adversarial Networks. Furthermore, \cite{carriere2020note} has examined optimization with such penalizations in the context of various applications. The novelty of the present work resides in the use of such models in the framework of a regression over a manifold and its joint utilization with a Laplace eigenbasis, enabling a deeper understanding of its topology. It is also the opportunity to study higher dimensional examples where the behavior of topological persistence is fundamentally different from the one of total variation. Indeed, we will see that topological persistence is a very convenient way to prevent the estimated functions from oscillating too much in a stronger way than more standard approaches.

The rest of the paper is organized as follows: in an intuitive fashion, Section~\ref{sec:motiv} presents the motivation behind the introduction of a topological penalty for a Laplace eigenbasis regression and how it can overcome the limitations of total variation denoising. Section~\ref{sec:theo} discusses two types of topological penalties: one is equivalent to solving a Lasso problem with weights and therefore has a simple theoretical analysis and even has a closed-form solution, while the other one is non-convex. Despite the lack of guarantees in non-convex optimization, we will present an oracle result for the estimated parameter. We also present a result on controlling the topology, or on topological sparsity, of one of our approaches. In Section~\ref{sec:expe}, we will present the results of experiments conducted on both synthetic and real data, in order to highlight the strengths and weaknesses of such an approach as opposed to standard regression methods. We have made the code used in several examples available here.\footnote{ \url{https://github.com/OlympioH/Lap_reg_topo_pen}}

\section{Motivation}
\label{sec:motiv}

\subsection{Laplace eigenbasis regression}
We study a regression problem on a compact, smooth submanifold  $\mathcal{M}$ of dimension $d$ of $\mathbb{R}^D$ without boundary. Throughout this paper, we assume the data points $(X_i)_{i=1}^n$ are sampled uniformly and independently over $\mathcal{M}$. Furthermore, for $i=1,\ldots,n$, the responses $Y_i$ are generated based on the model 
\begin{align}\label{eq:datagm}
 Y_i = f^\star(X_i) + \varepsilon_i,
\end{align}
 where $(\varepsilon_i)_{1\leq i \leq n}$ are i.i.d. zero-mean sub-Gaussian noise variables independent of all the $X_i$'s. Our goal is to retrieve the function $f^\star$, also referred to as the regression function, from the given observations $(X_1, Y_1), \ldots, (X_n, Y_n)$.
 
A natural choice of basis to perform a regression and exploit the manifold-structure assumption is the Laplace-Beltrami eigenbasis. Analogously to the Euclidean case, the Laplace-Beltrami operator $\Delta$ is (the negative of) the divergence of the gradient: $\Delta f = -\nabla \cdot \nabla f$. If we denote by $g$ the metric tensor and by $g^{ij}$ the components of its inverse, we have the following expression in local coordinates (with Einstein summation convention):
\[ \Delta f = -\frac{1}{\sqrt{\det (g)\;}} \partial_i (\sqrt{\det (g)\;} g^{ij} \partial_j f).
\]
We remark here that due to our uniform sampling assumption on the $X_i$, if suffices to consider the standard Laplace-Beltrami operator as above. The methodology and theory we develop will immediately extend to non-uniform sampling schemes based on \textit{weighted} Laplace-Beltrami operators~\citep{rosenberg1997laplacian, grigoryan2009heat} as long as the sampling distribution is sufficiently light-tailed (say, it satisfies Poincar\'e inequality). In order for our exposition to convey our main contribution on topological penalization, we stick to the uniform sampling assumption in the rest of this paper.  

Notice that in the Euclidean case, where $g$ is the identity matrix, we retrieve the usual well-known formula for the Laplacian (up to a sign convention). The operator $\Delta$ is a self-adjoint operator with compact inverse which implies that its set of eigenvalues is discrete and that they all are non-negative~\citep{rosenberg1997laplacian}. We can then sort the eigenvalues $(\lambda_j)_{j \geq 1}$ in nondecreasing order and approximate $f^\star$ as a linear combination of the corresponding normalized eigenfunctions $(\Phi_j)_{j \geq 1}$. Besides being an orthonormal basis of $L^2 (\mathcal{M})$ with many smoothness properties, it is a known fact that the functions $(\Phi_j)_{j \geq 1}$ are related to the topology of the manifold~\citep{zelditch2009local}. In addition, the Laplace-Beltrami eigenbasis can be seen as an extension of the Fourier basis to general manifolds. Indeed, on the two dimensional flat-torus $\mathbb{R}^2/2 \pi \mathbb{Z}^2$, the eigenvalues of the Laplace-Beltrami operator are $(n^2 + m^2)_{m, n \in \mathbb{N}}$ and possible corresponding eigenfunctions are $(x, y) \mapsto \sin (nx) \sin (my)$ up to a normalization constant (see for instance Chapter 4.3 of \citealp{zelditch2017eigenfunctions}). By analogy with the approximation of a function by its truncated Fourier series in classical analysis, it is natural to choose the eigenfunctions corresponding to the $p$ smallest eigenvalues as a suitable expansion basis for the signal.

Once the number $p$ of features is chosen, the problem boils down to using the observed data for finding $\theta \in \mathbb{R}^p$ such that $\sum_{i=1}^p \theta_i \Phi_i$  is a good approximation to $f^*$. To this end, we introduce the design matrix $\mathbf{X} \in \mathbb{R}^{n \times p},$ where $\mathbf{X}_{ij}=\Phi_j(X_i),$ and we let $\hat{\theta}$ be a minimizer of
\begin{equation}
    \label{eq:penloss}
\mathcal{L}(\theta) = \|Y-\mathbf{X} \theta\|_2^2 + \mu \Omega (\theta),
\end{equation}
where $Y=(Y_1, \ldots, Y_n)$ is the response vector and $\Omega$ is a penalty term also depending on the Laplace-Beltrami eigenfunctions. Our choices for $\Omega$ will be discussed below. The scalar $\mu$ is a calibration factor aiming at reducing overfitting. In case we do not know the eigenfunctions $\Phi_i$, we will use eigenvectors of a graph Laplacian as sample approximation (see below for details). 

Examples of classical penalties include $L^1$-regularization, also called Lasso (see~\citealp{buhlmann2011statistics} for an exhaustive treatment), and total variation penalty. Although the latter provides good theoretical guarantees (see, for example~\citealp{hutter2016optimal} for oracle results and \citealp{dutta2018covering} for a metric entropy based approach), it fails to capture some aspects of the geometry of the data. Indeed, consider the square $[0, 1]^2$ (or equivalently the 2D torus), discretize it as small squares of size $\varepsilon$ (we can assume $\varepsilon$ to be equal to $1/N$ for some integer $N$ to avoid boundary issues) and consider a pyramidal function $f_\varepsilon$ on each square with value 0 at the boundary of the square, and a maximum of $\varepsilon$ attained in the middle of the square (see Figure \ref{pyramids}). The total variation of the so obtained function is equal to $\sum_{\text {cells}} \int |\nabla f_\varepsilon| = \sharp \text {cells} \int_{\text{cell}} |\nabla f_\varepsilon|$. Since $ |\nabla f_\varepsilon| = 2$, it yields that $ TV(f_\varepsilon) = 2$. In particular, it does not depend on $\varepsilon$, which means that total variation is blind to very small perturbations of the function and is therefore not suited to deal with such a type of noise. We are now going to see in the following subsection a type of penalty that can capture such small oscillations.

\begin{figure}[!h]
\begin{center}
\includegraphics[scale=0.5]{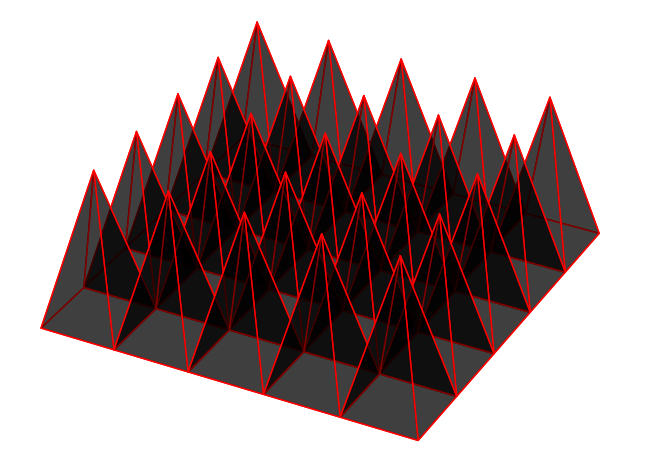}
\end{center}
\caption{Pyramidal function $f_\varepsilon$.}
\label{pyramids}
\end{figure}

\subsection{Total Persistence}
In this section, we present the most basic concepts of topological data analysis as introduced in the reference textbooks by \cite{edelsbrunner2010computational} and \cite{boissonnat2018geometric}. We will try to keep the notions as intuitive as possible and do not lay out the technical details of homology theory. Consider the sub-level sets $\{ f^{-1}((- \infty, t]) \}_{t= - \infty}^{ + \infty}$ of a given function $f$. As $t$ traverses $\mathbb{R}$ from $- \infty$ to $+ \infty$, the topology of the sub-level sets changes and we keep track of these changes in the so-called persistence diagram. More precisely, suppose we are interested in $k$-dimensional topological features present in a sub-level set (or in a topological space), namely a connected component for $k=0$, a cycle for $k=1$, a void for $k=2$, and so on. For simplicity, assume that $f$ is a Morse function (in particular, its critical points are non-degenerate), such that the topology of the sub-level sets of $f$ only changes at levels $t$ corresponding to extremal points (see Theorem~\ref{Morse} below). Then, as the level $t$ increases, such $k$-dimensional features might start to exist at a certain level $t_b$ and they might disappear by merging with another component at a different level $t_d$, where $t_d$ might be equal to $+\infty$ if it never disappears. Then we place a point in the plane with coordinates $(t_b, t_d)$. The set of all such points (each corresponding to a different $k$-dimensional feature) along with the diagonal $y=x$ (accounting for the fact that in general features might appear and disappear at the same level or time) forms the $k-$th persistence diagram of $f$. It is a multi-set of $\mathbb{R}^2$ as different features might appear and disappear at the same levels. The example of a persistence diagram for a one-dimensional function shown in Figure \ref{PDfunc} is taken  from~\cite{boissonnat2018geometric}.

\begin{figure*}[t]
\begin{center}
\includegraphics[width=0.9\linewidth]{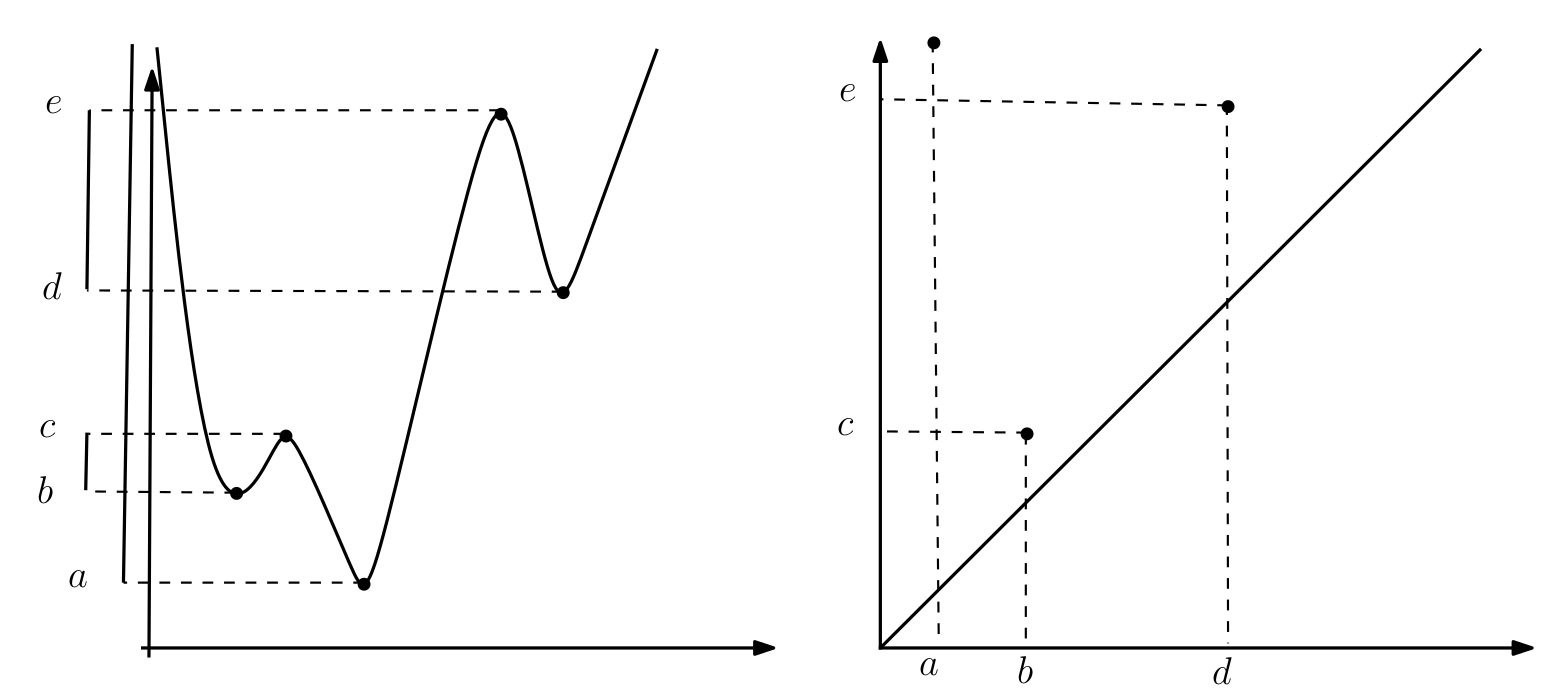}
\end{center}
\caption{Persistence diagram of a real-valued function.}
\label{PDfunc}
\end{figure*}

For the persistence diagram of a function $f$ to be well-defined, we need the function to satisfy a tameness assumption \citep{edelsbrunner2010computational}. Sufficient for this is to assume $f$ to be a Morse function. The following result makes precise the above mentioned fact that for Morse functions the topology of the sub-level sets of $f$ can be simply described in terms of its critical points, defined by $\nabla f(x) = 0$. Recall that critical points of Morse functions are non-degenerate (non-singular Hessian), and that the index of a critical point is the number of negative eigenvalues of the Hessian.

\begin{theorem}[\cite{edelsbrunner2010computational}]
\label{Morse}
Let f be a Morse function on a smooth manifold $\mathcal{M}$ and denote by $\mathcal{M}^a$ the sub-level set  $f^{-1} ( - \infty, a]$.

\begin{itemize}
\item Suppose that there is no critical value between $a < b$. Then $\mathcal{M}^a$ and $\mathcal{M}^b$ are diffeomorphic and $\mathcal{M}^b$ deformation retracts onto $\mathcal{M}^a$.
\item Suppose $p$ is a non-degenerate critical point of $f$ with index $s$ and that $f(p)=q$. We further assume there are no other critical points $p^\prime$ with $f(p^\prime)=q$. Then for $\varepsilon$ small enough, $\mathcal{M}^{q+\varepsilon} $ is homotopy equivalent to $\mathcal{M}^{q-\varepsilon}$ with a $s$-handle attached.
\end{itemize}
\end{theorem}

As a consequence, for Morse functions all the coordinates of the points in the persistence diagrams of every dimension are critical values of $f$. Furthermore, in the persistence diagram of a feature dimension $k$, the birth times are critical values of index $k$ and the death times are critical values of index $k+1$. We define the persistence of a feature to be its lifetime, namely its death time $t_d$ minus its birth time $t_b$, and define the $k-$persistence of a function, denoted by $\chi_k (f)$, as the sum of all individual persistences in dimension $k$.  In the literature this sometimes is also called the $k$-total persistence. When we talk about the persistence of a function $\chi(f)$, it is understood to be the sum of all persistences over all dimensions. The count of all the $k$-dimensional features in a topological space (sub-level set) is called the $k$-th Betti number of this space. The total sum of all the $k$-th Betti number is called the total Betti number of this space.

Note that the existence of features with infinite persistence would make $\chi(f)$ equal to $\infty$. To avoid this degeneracy, the quantity $\chi(f)$ is modified (see \citealp{polterovich2019topological}) by replacing the infinite persistence of a feature born at $b$ by $\max(f) - b.$ The number of  topological features with infinite persistence equals the total Betti number $\zeta = \zeta({\cal M})$ of the manifold $\mathcal{M}$. In what follows, we will always consider persistences to be clipped as such. We also state a useful result from Chapter 6 of \cite{polterovich2019topological} in Lemma~\ref{stability} below. It can be seen as a corollary of the famous stability inequality in topological data analysis from \cite{cohen2007stability}. The result essentially states that two functions close in uniform norm necessarily have close persistence. 

\begin{lemma}[\citealp{polterovich2019topological}]
\label{stability}
Let $f$ and $h$ be two Morse functions on a manifold $\mathcal{M}$ with total Betti number $\zeta$. 
Denote by $\nu(f)$ the total number of points (with finite persistence) in the persistence diagram of $f$. Then
\[  \chi(f) - \chi(h) \leq (2 \nu (f) + \zeta) \|f-h\|_\infty.
\]
\end{lemma}
This result remains true when $f$ is Morse and $h$ is only continuous. Under those circumstances, $\chi(h)$ can be defined by the (possibly infinite) limit of the total persistence of a sequence of Morse functions that uniformly converges towards $h$, as done in \cite{polterovich2019persistence}. It is worth mentioning here that more precise stability results for difference of total persistences with respect to the $L_1$ metric (instead of $L_\infty$) are available in \cite{skraba2020wasserstein}, in the case where functions are defined on top of CW-complexes. Though adaptation of such results to sub-level sets based filtration seems possible, applications to this particular case of regression on manifold would lead to the same kind of bounds. Indeed, Lemma \ref{sup_norm} ensures that sup-norm bounds are of the same order as $L_1$ bounds in this case. Nonetheless, we believe that substantial gains might be expected from using these refined bounds in more general regression settings.

\section{Methodology}
In applications we construct persistence diagrams from random data --- think of a  random function, such as an estimated regression function, or a function with noise added; see below. The standard paradigm in topological data analysis is that in such random persistence diagrams the features with a high persistence are true features, whereas the features with a low persistence that lie near the diagonal are noisy perturbation of the topology. We denote that recent results from \cite{bubenik2020persistent} are changing this paradigm since relevant topological information can be found in low-persistence features. Though it is likely that some local information may be retrieved from these small persistence features (such as geometrical characteristics of the support), in a regression setting given a noisy input, topological smoothness of the regression function is enforced via discarding these small oscillations. We propose two penalization strategies that intend to achieve this goal. We can see an example of the influence of noise on the persistence diagram Figure~\ref{gaussian_data} where we have computed the value of a function at $1000$ points uniformly sampled in the square $[0, 10]$ and where we have added Gaussian noise to each entry with three different levels $(\sigma = 0.01, \sigma = 0.05, \text{ and } \sigma = 0.1$), and then plotted an interpolation.  The function considered here is the sum of four Gaussian functions on a square. This function has a single topological feature of dimension 0 (a connected component is born at level 0 and never dies) and four topological components of dimension 1 (that die at the height of the local modes of each Gaussian). When adding noise to this function, the resulting persistence diagram has many points and the noisy function has a very large persistence.  The higher the noise level, the further the noisy features are from the diagonal, until it is hard to distinguish the four true topological features from the noisy features. We can see this observation reflected in the plots of the function itself. This motivates to consider methods that sparsify the persistence diagram in order to denoise the input.

\begin{figure*}
\begin{center}
\subfigure[Original function]{\includegraphics[scale=0.5]{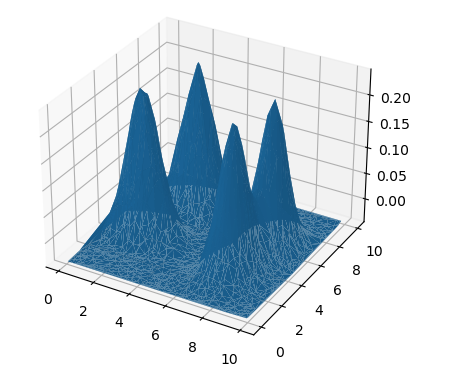}}
\subfigure[$\sigma = 0.03$]{\includegraphics[scale=0.3]{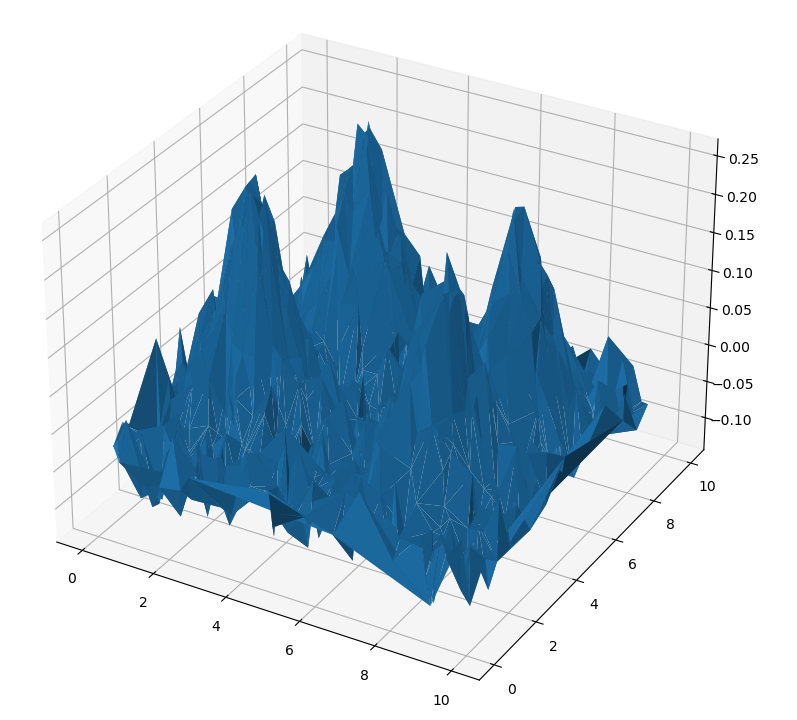}}\\
\subfigure[$\sigma = 0.05$]{\includegraphics[scale=0.3]{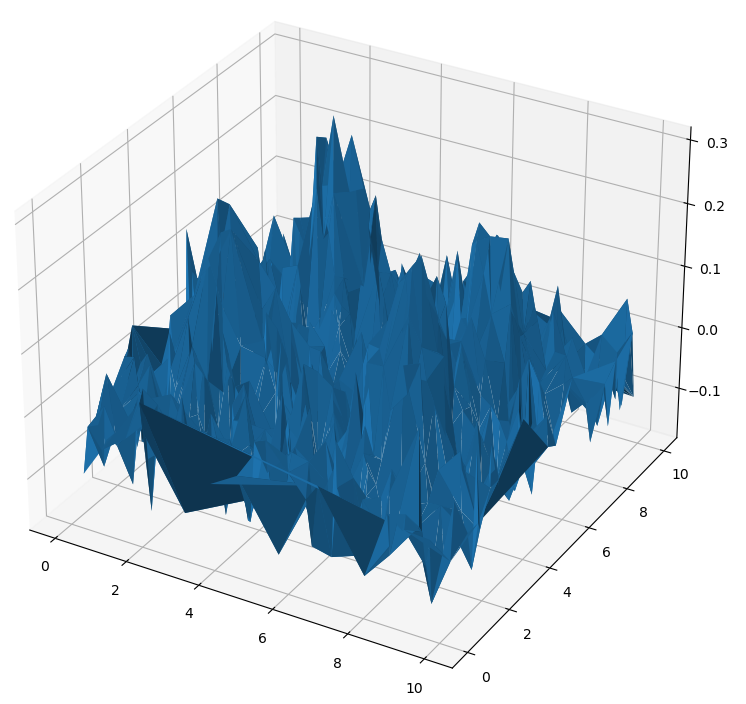}}
\subfigure[$\sigma = 0.1$]{\includegraphics[scale=0.3]{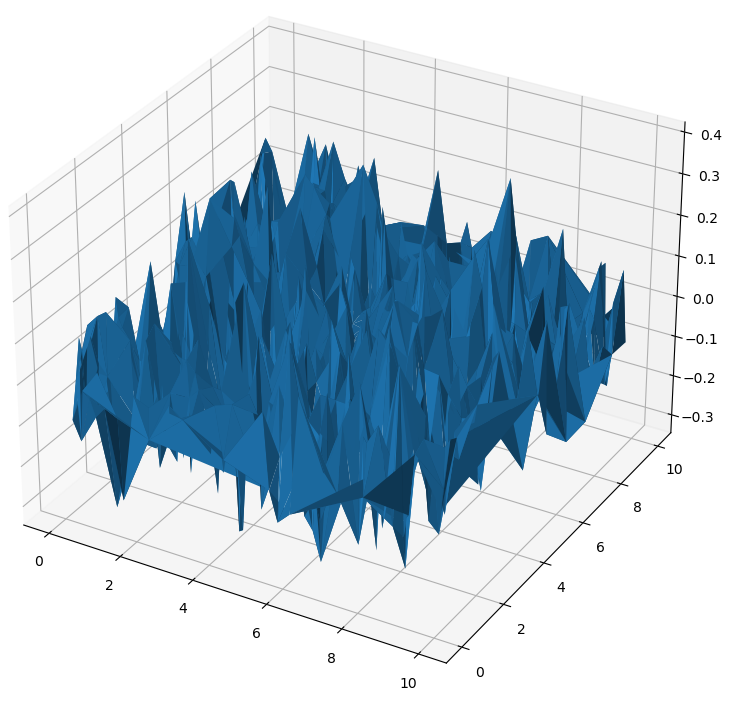}}\\
\subfigure[Original function]{\includegraphics[scale=0.17]{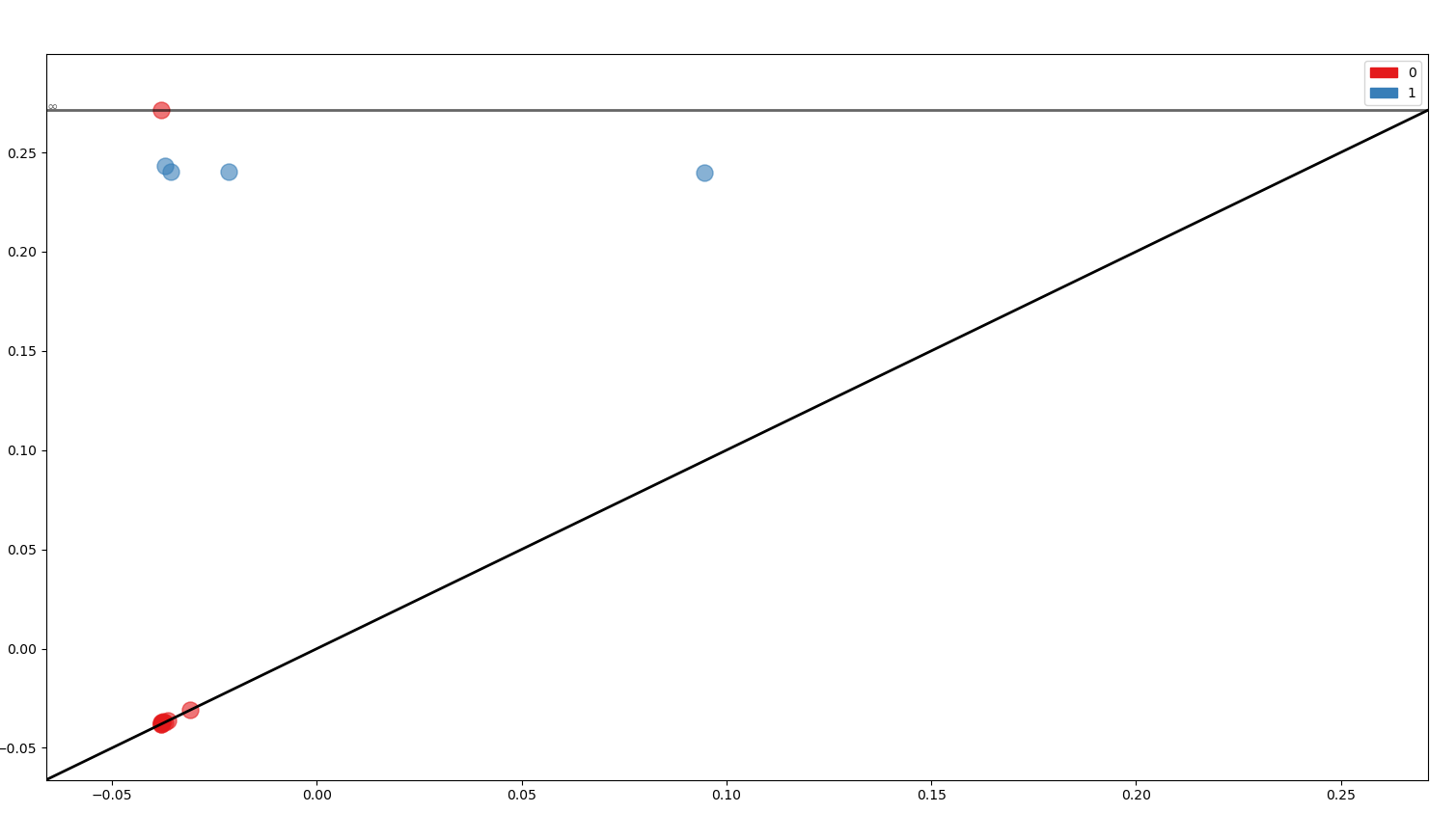}} 
\subfigure[$\sigma = 0.03$]{\includegraphics[scale=0.17]{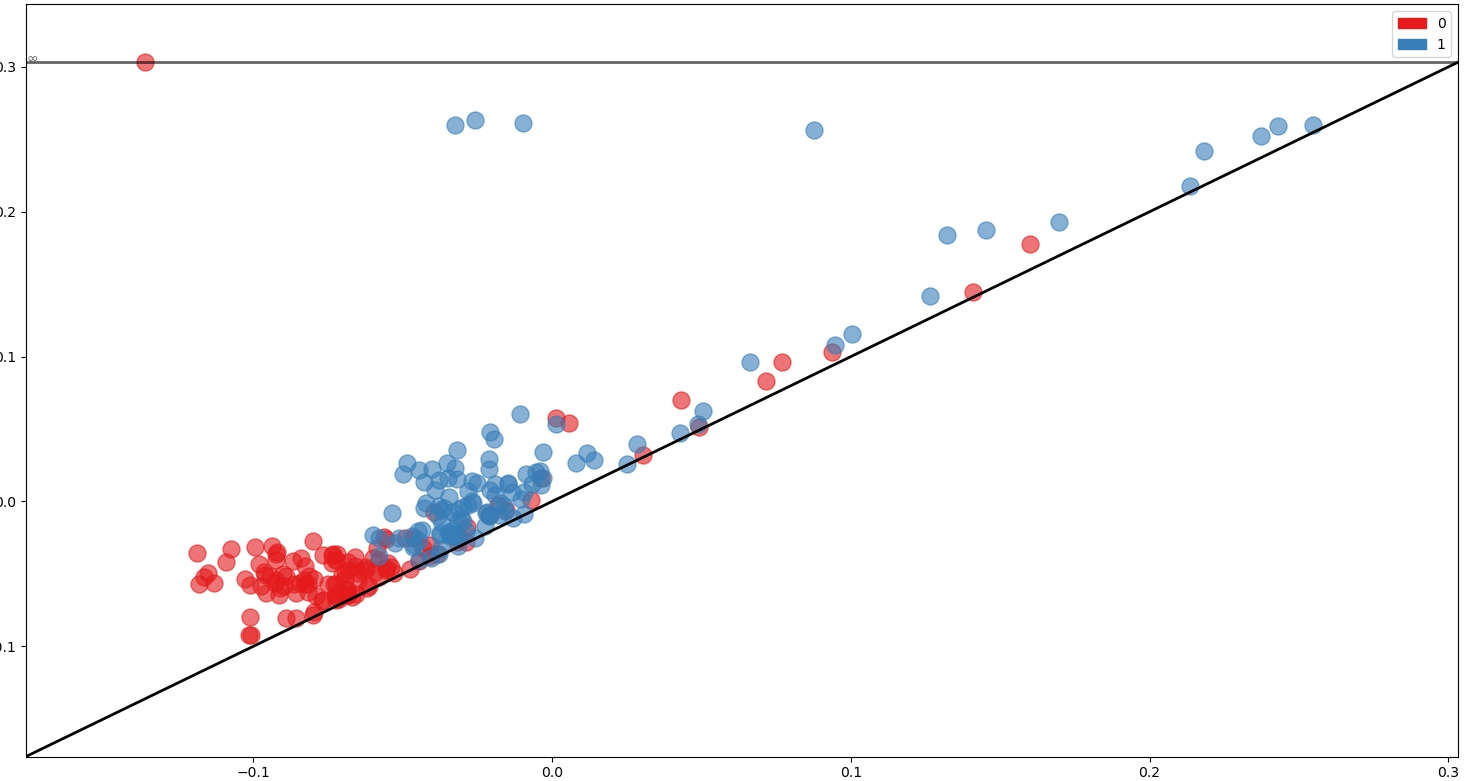}} \\
\subfigure[$\sigma = 0.05$]{\includegraphics[scale=0.17]{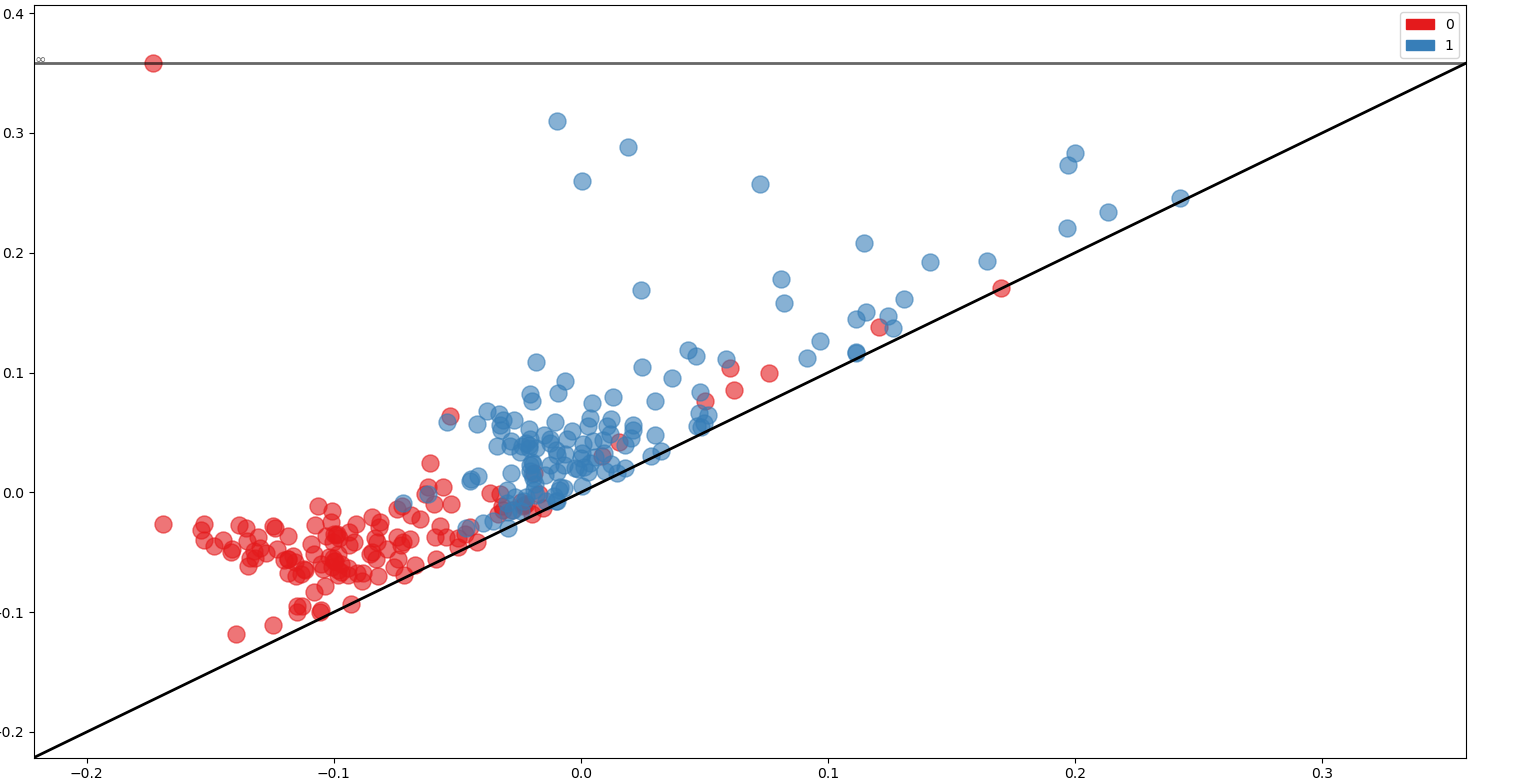}} 
\subfigure[$\sigma = 0.1$]{\includegraphics[scale=0.17]{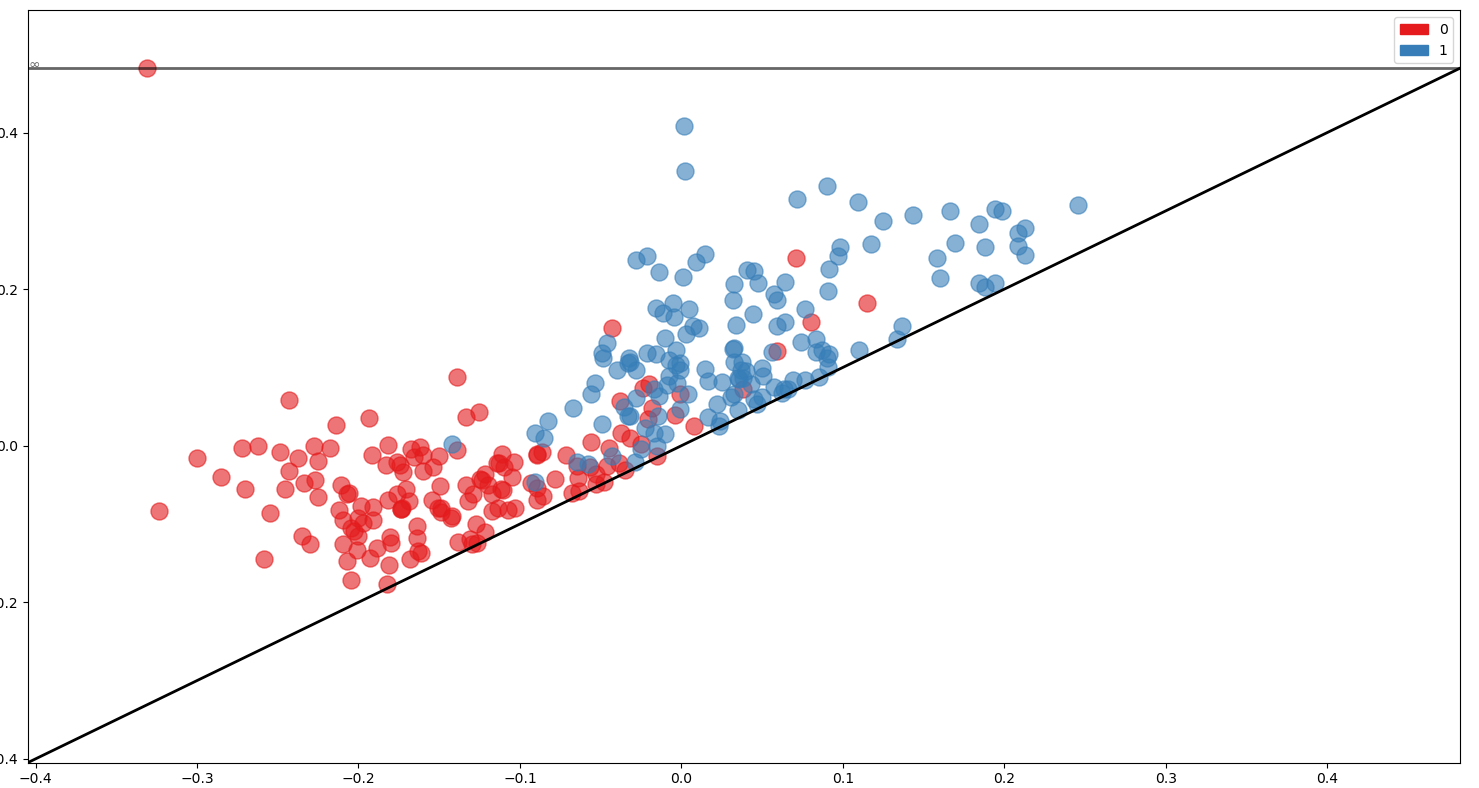}}
\center\caption{Influence of noise on persistence diagrams.}
\label{gaussian_data}
\end{center}
\end{figure*}

\subsection{Two types of penalties}
\label{two_types}

We will introduce two different ways of penalizing the persistence. The first one aims at reducing the dimension of the problem by selecting a `small' number of eigenfunctions, while the second one is more focused on denoising and providing a smoother output.

We first consider the penalty 
\begin{equation}
    \label{eq:omega1}
\Omega_1 (\theta) = \sum_{i=1}^p |\theta_i| \chi (\Phi_i),
\end{equation}
where the $(\Phi_i)_{i=1}^p$ are eigenfunctions of the Laplace-Beltrami operator. From a theoretical viewpoint, every homological dimension should be penalized in order to capture every possible oscillation of the regression function. To give an illustration, considering the pyramidal oscillations of the function depicted in Figure \ref{pyramids2} upward oscillations are captured by homology of dimension one, whereas downward ones may be seen on the 0-dimensional persistence diagram. Depending on the problem at hand, the persistence of only one or a few chosen homological dimensions can be penalized as we will see in Section \ref{sec:expe}. When treating high-dimensional data, it actually becomes a computational necessity to only penalize by the first homological dimensions, as we will discuss in Section \ref{complexity}. The idea behind the penalty is that the more a function oscillates, the more likely it is to overfit the data.  The penalty $\Omega_1$ can be understood as a weighted Lasso penalty, with weights being the persistences of each eigenfunction. The weighted Lasso is a broadly studied model (see~\citealp{buhlmann2011statistics} for an exhaustive reference). It induces sparsity in the representation of the function, and in our context it aims at introducing an inductive bias towards discarding eigenfunctions with a large persistence.

The second persistence based penalty considered here is
\begin{equation}  \label{eq:omega2} \Omega_2 (\theta) = \chi \left( \sum_{i=1}^p \theta_i \Phi_i \right).
\end{equation}
While $\Omega_2$ does not induce sparsity over the parameter $\theta$, it aims at inducing a certain kind of `topological sparsity'. Indeed, the goal of this penalty is for the reconstructed function to have a much smaller number of points in the persistence diagram than the noisy function. Intuitively, this is achieved by optimizing $\theta$ so as to cancel out oscillations of the individual eigenfunctions $\Phi_i$ when summing them up. A main computational issue with the penalty $\Omega_2$ is its non-convexity. Indeed, if we consider the two functions $\cos(nx)$ and $\cos(ny)$, they both have a persistence of order $n$, yet the sum of the two has a persistence of order $n^2$. However, a result in \cite{carriere2020note} shows the convergence of a stochastic gradient descent algorithm towards a critical point of the loss function for the penalty $\Omega_2$.


Finally, we remark that one can potentially consider variants of penalty $\Omega_1$, for instance, a weighted Ridge penalty of the form $\sum_{i=1}^p \theta_i^2 \chi( \Phi_i)$. Preliminary experiments have shown that the weighted Lasso performs better than the weighted Ridge. In addition, as we will see in Section~\ref{sec:expe}, in applications a combination of the two penalties described above is quite efficient, and we actually benefit of the selection properties of the weighted Lasso.

\subsection{Contrasting persistence and total variation}\label{sec:tvvstp}

While persistence at a first glance might appear to be a measure of the regularity of a function similar to total variation, this only is true for a function in one dimension, where the persistence is half the total variation for functions on the circle $\mathbb{S}^1$ (see \citealp{polterovich2019topological}). In higher dimensions these two penalties are no longer equivalent as discussed in the following. 

A first indication of the differences between total variation and persistence is given in Figure~\ref{pers_torus}, showing  Laplace-Beltrami eigenvalues on the flat torus along with persistences and total variations of their eigenfunctions $\sin(nx) \sin(my)$. Note that for this figure, the persistence and the total variation have been computed numerically for eigenfunctions defined on a regular grid. Within an eigenspace, the eigenvalues are sorted in lexicographical order on $(n, m)$. We defer to Section \ref{sec:expdesign} for more details on how to numerically compute persistences. We remark here that the $x$-axis (corresponding to the index of eigenfunctions) in the sub-figures are all aligned.

\begin{figure*}[t]
\begin{center}
\subfigure[Eigenvalue]{\includegraphics[scale=0.11]{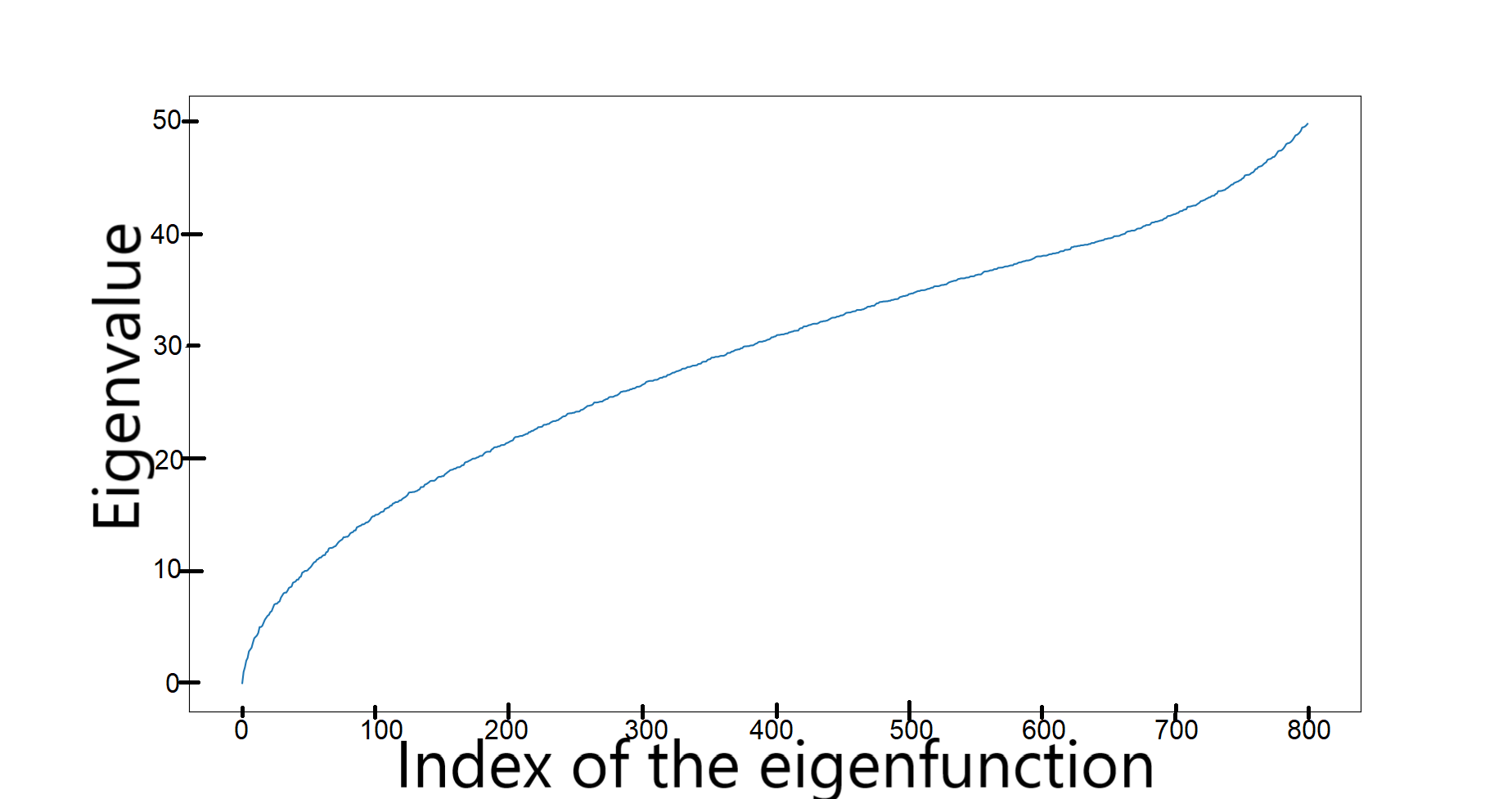}}
\subfigure[Total variation]{\includegraphics[scale=0.14]{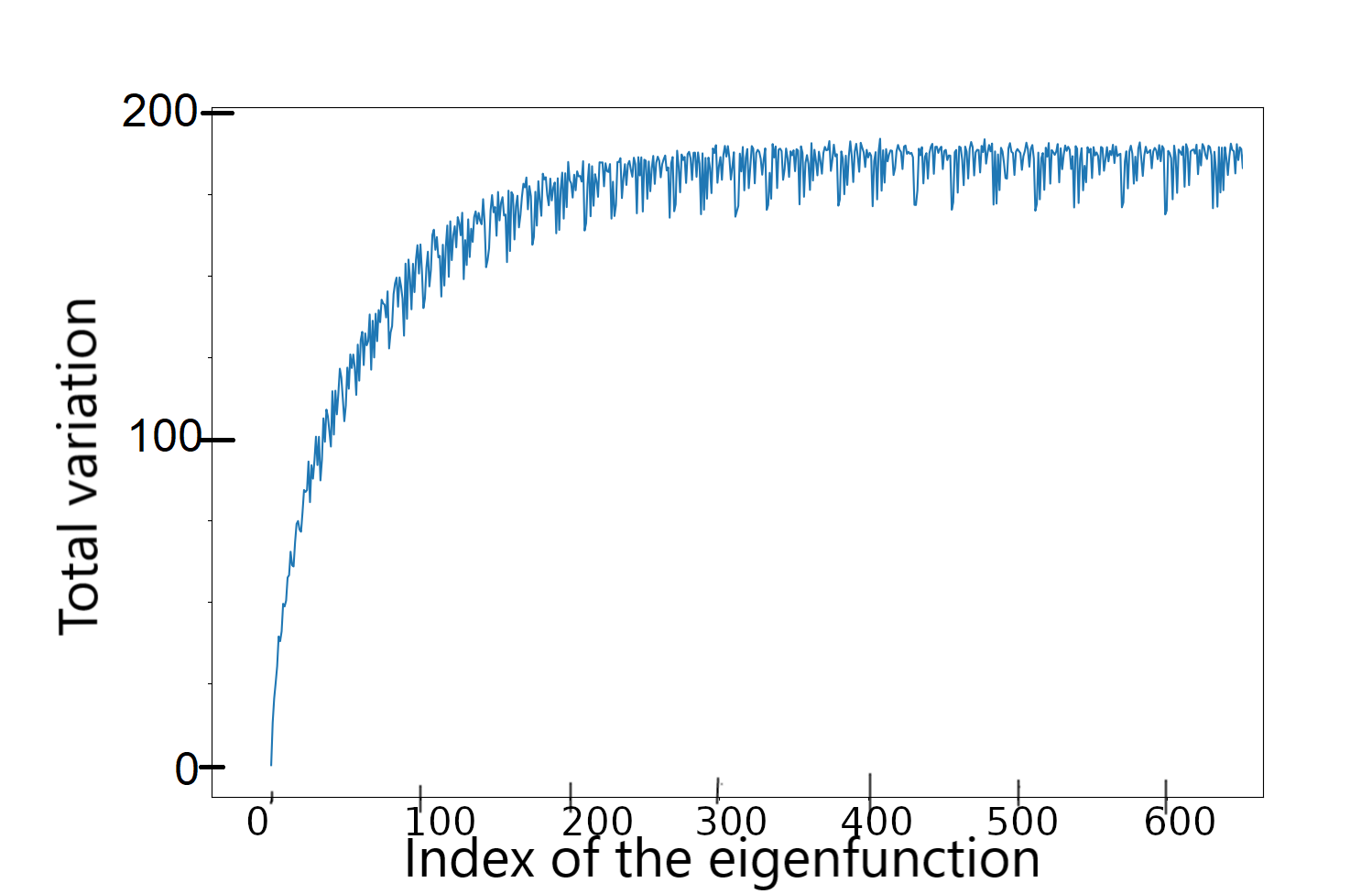}} 
\subfigure[Persistence]{\includegraphics[scale=0.14]{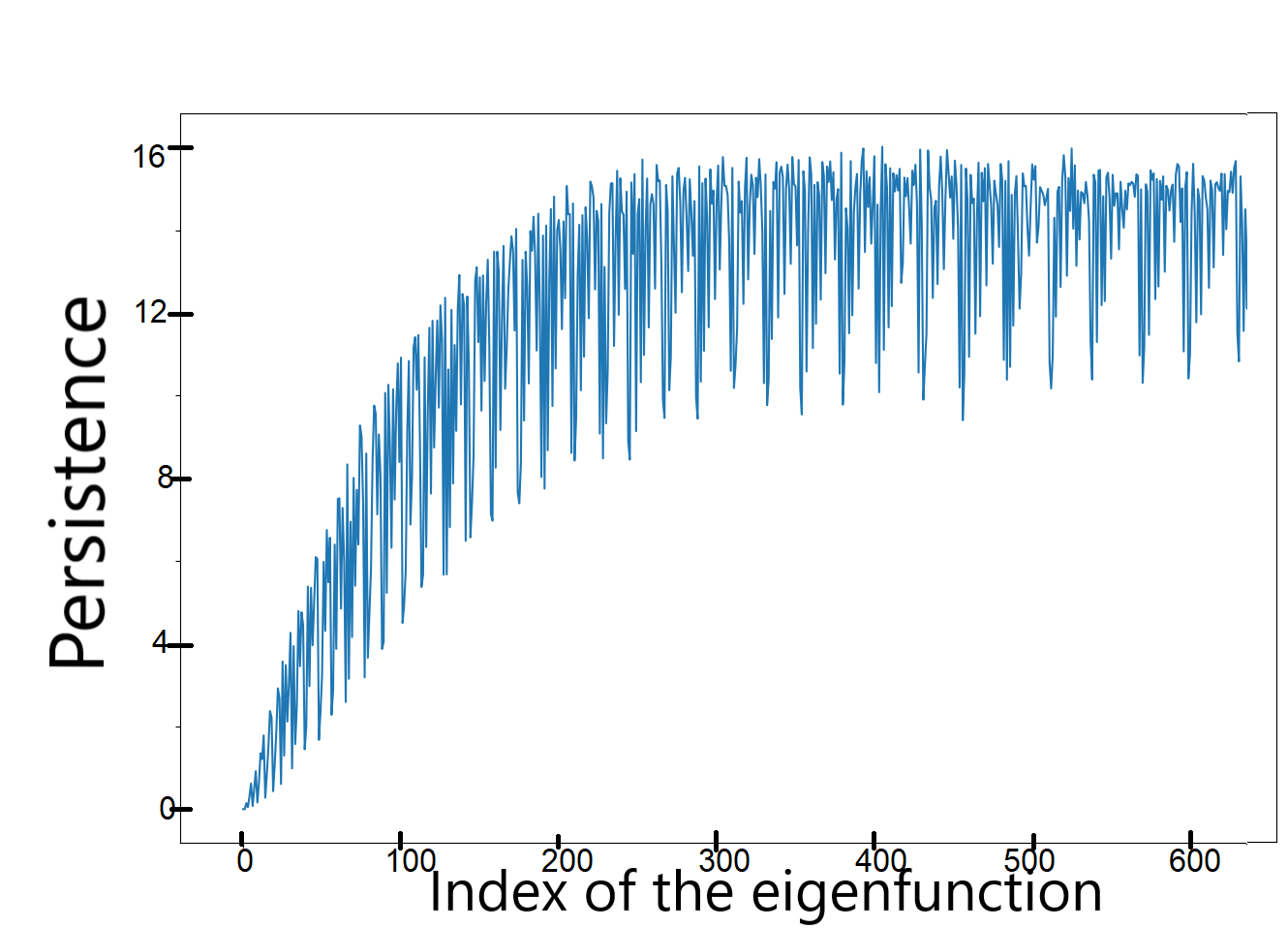}} 
\caption{Various quantifiers of the oscillations of the eigenfunctions of the Laplacian.}
\label{pers_torus}
\end{center}
\end{figure*}

 Figure~\ref{pers_torus} shows that the eigenvalues  increase `smoothly' and using them as weights for a Lasso-type penalty is a way to regularize the oscillatory behavior of eigenfunctions of large index.  However, it can be seen in panel (c) in Figure~\ref{pers_torus} that while the persistences of the eigenfunctions show an increasing trend with increasing eigenvalues, the persistences also show an overlaid periodic behavior. A similar behavior can be seen for the total variation, but with a much smaller periodic effect. The significant periodic behavior of the persistences means that eigenfunctions can have similar persistences, even if their eigenvalues are quite different. Vice versa, eigenfunctions with similar (or equal) eigenvalues can have quite different persistences. Indeed, for even fixed integers $n$ and $m$, let $\Phi: (x,y) \mapsto \sin(nx) \sin(my)$ be a corresponding eigenfunction. Its gradient is $\nabla \Phi (x, y) = (n \cos(nx) \sin(my), m \sin(nx) \cos(my))$ and it therefore has $2nm$ critical points. By using the fact that the number of saddle points must be equal to the sum of the number of maxima and the number of minima because the Euler characteristic of the torus equals 0, we obtain that $\Phi$ has exactly $nm/2$ maxima, $nm/2$ minima and $nm$ saddles. One of the minima, two of the saddle points, and one of the maxima generate essential homology classes whose corresponding persistent homology classes live forever. Following the convention taken, those are truncated at the maximum value of the function. The persistence diagram of dimension 0 has a point of persistence 2 and $mn/2 - 1$ of persistence 1. Therefore, the 0-persistence is $mn/2 + 1$. For homological dimension 1, we have $mn/2 + 1$ points, all of them have persistence 1, so the 1-persistence is $mn/2 + 1$. For homological dimension 2, the only point in the persistence diagram has coordinates $(1, 1)$. The total persistence is therefore $mn+2$. The case where $n$ or $m$ is odd is very similar and also yields that the persistence is of order $mn$. This means that within an eigenspace with eigenvalue $\lambda = n^2 + m^2$, a penalty on the persistence is proportional to $mn$, therefore eigenfunctions $\sin (nx) \sin(my)$ with $n$ or $m$ small are more likely to be kept in the model. For instance, the eigenfunctions $ (x,y) \mapsto \sin(10x)\sin(y)$ and $(x,y) \mapsto \sin(8x) \sin(6y)$ correspond to eigenvalues $101$ and $100$ but have very different persistence, namely five times larger for the latter eigenfunction. This effect is much less pronounced for the total variation penalty. 


While the above already shows some differences between the two types of measures of regularity of functions, the following observation is perhaps even more relevant for our purposes. To this end, let us reconsider the example of the 2-dimensional pyramidal function $f_\varepsilon$ shown in  Fig.~\ref{pyramids}. We already observed above that the TV-penalty does not depend on the choice of $\epsilon.$  To understand the behavior of the persistences of these functions, observe that the sub-level sets of the function $f_\varepsilon$ are empty for levels $t < 0$, and then for $t \in(0,  \varepsilon)$, the sub-level sets have $1/\varepsilon^2$ homology components of dimension 1, that all merge at $\varepsilon$. Therefore, the 1-persistence diagram of $f_\varepsilon$ has $1/\varepsilon^2$ points, all born at level 0 and dying at time $\varepsilon$. Thus, $f_\varepsilon$ has a 1-persistence of $1/\varepsilon$, which thus increase to infinity as $\epsilon \to 0$. This is in stark contrast to the behavior of the total variation. In a similar fashion, the sub-level sets of the function $-f_\varepsilon$ have $1/\varepsilon^2$ connected components from $-\varepsilon$ to $0$ that all merge at 0. Therefore, $\chi_0 (-f_\varepsilon) = 1/\varepsilon$ while it also has a total variation that does not depend on $\varepsilon$. 

\begin{figure}[t]
\begin{center}
\includegraphics[scale=0.4]{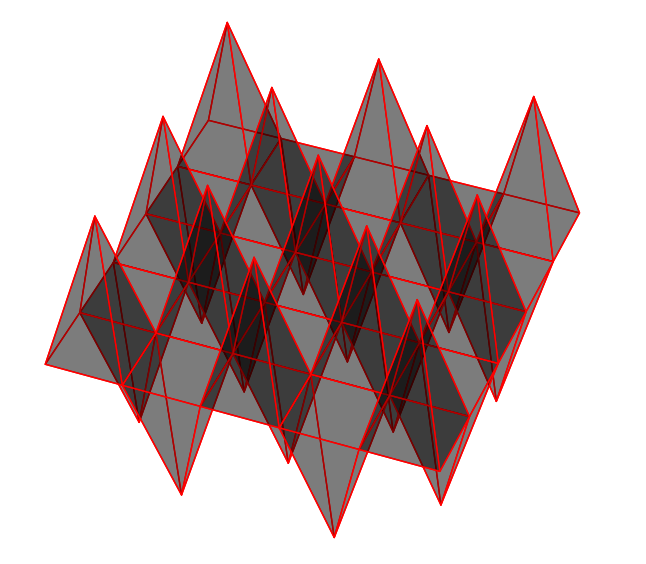}
\end{center}
\caption{Alternating pyramids.}
\label{pyramids2}
\end{figure}

\begin{wraptable}{r}{7.5cm}
\centering
\begin{tabular}{|c|c|c|c|}
	\hline
$\| \cdot \|_\infty$ & Lip & TV & $\chi_1$ \\
	\hline
$ \varepsilon$ & 2 & 2 & $1/ \varepsilon$ \\
	\hline
\end{tabular}
\caption{Quantities characterizing $f_\varepsilon$.}
\label{tab_pyramids}
\end{wraptable} 

As an example of a function where both $\chi_0$ and $\chi_1$ are of importance, consider the same discretization of the space as above, where this time we alternate between a pyramid of height $\varepsilon$ and a reversed pyramid of height $-\varepsilon$ (see Figure~\ref{pyramids2}). Similarly to the two previous cases, the persistence diagram of this function has $1/(2\varepsilon^2) $ points of coordinate ($-\varepsilon$, 0) (for the 0-homology) and $1/(2\varepsilon^2) $ points of coordinate (0, $\varepsilon$) (for the 1-homology). Therefore, its 0-persistence is equal to its 1-persistence, both equal to $1/(2\varepsilon),$ while the total variation here again does not vary with $\varepsilon$.

It is straightforward to build similar examples in higher dimensions and the effect will be even more striking: For instance, discretize the $d-$dimensional hypercube with cubes of size $1/\varepsilon^d$ and consider a function increasing linearly towards the center of each cube until it reaches a maximum of $\varepsilon$. Such a function still has a total variation of order 1 no matter the value of $\varepsilon$, however, its $(d-1)$-persistence will be equal to $1/\varepsilon^{d-1}$.

The takeaway of the above discussion is as follows: Consider Table \ref{tab_pyramids} which provides a summary of various measures of regularity of the pyramidal function $f_\varepsilon$. We see that if we penalize the supremum norm of this function, the penalty will have no effect as $\varepsilon \to 0$. If we try to penalize its Lipschitz constant or its total variation, the penalty will be the same, no matter the scaling $\varepsilon$ and it will therefore have a very limited effect. In contrast to that, when penalizing the persistence, the effect of the penalty will become quite important as $\varepsilon$ becomes small. Such a function $f_\varepsilon$ is assimilated to noise as $\varepsilon \to 0$ and we want to penalize it as much as possible, which is something that can be achieved by persistence but not by total variation.

\subsection{Complexity of functions with bounded persistence: A negative result}
When penalizing persistence, a natural question that immediately arises is to measure the \textit{complexity} of the set of bounded-persistence functions. 
Loosely speaking, if the set of candidate functions has a large complexity,  
seeking for a candidate function (e.g. minimizing an empirical loss) can be very challenging and furthermore, 
the control of the excess risk 
between $f^\star$ and its estimation becomes non-informative (see \citealp[Sections 3 and 11]{mohri2018foundations} for a more detailed exposition). For regression problems, a standard measure of the size (complexity) of a set of functions $\mathcal{F}$ is the so-called fat-shattering dimension introduced in \cite{kearns1994efficientlearn}. 
\begin{definition}
Let $\gamma >0$. A set of points $\mathbb{X}=\{X_1, \ldots, X_l\}$ is said to be $\gamma$-shattered if there exists thresholds $r_1, \ldots, r_l$ such that for any subset $E \subset \mathbb{X}$, there exists a function $f_E \in \mathcal{F}$ such that $f_E(x_i) \geq r_i+\gamma$ if $x_i \in E$ and $f_E(x_i) < r_i - \gamma$ if $x_i \notin E$ for all $i$. The fat-shattering dimension $\mathrm{fat}_\gamma(\mathcal{F})$ of the class $\mathcal{F}$ is then equal to the cardinality of the maximal $\gamma$-shattered set $X$. 
\end{definition}

Note that the fat-shattering dimension of the class $\mathcal{F}$ depends on the parameter $\gamma > 0$. A class $\mathcal{F}$ has infinite fat-shattering dimension if there are $\gamma$-shattered sets of arbitrarily large size. It is well-known that bounds on the fat-shattering dimension lead to bounds on the covering number and hence the metric entropy and Rademacher complexity of the function class. Furthermore,~\cite{bartlett1996fat} showed that a function class $\mathcal{F}$ is learnable (in the sense of~\citealp[Definition 2]{bartlett1996fat}) if and only if it has finite fat-shattering dimension. Unfortunately, in the case of bounded persistence functions, we have the following result:

\begin{theorem}
\label{theo:fat}
Let $\mathcal{H}_V= \{ f: [0, 1]^d \to [0, 1] |\, \chi(f) \leq V \}.$
Let $0 < \gamma <1/2$.
\begin{itemize}
\item If $d=1, \quad \mathrm{fat}_\gamma (\mathcal{H}_V) \leq 1+\lfloor \frac{V}{\gamma} \rfloor$.
\item If $d \geq 2$, then $\mathrm{fat}_\gamma (\mathcal{H}_V) = \infty$ if $2 \gamma \leq V,$ and $\mathrm{fat}_\gamma (\mathcal{H}_V) = 1$ otherwise.
\end{itemize}

\end{theorem}

\begin{proof}
First we note that if $d=1$, for any function $f : [0, 1] \to [0, 1]$, we have that $\chi(f) = \frac{1}{2} (TV(f)+|f(1)-f(0)|)$, and therefore:
\[\frac{1}{2}TV(f) \leq \chi(f) \leq TV(f).
\] 
We can therefore derive the claim for $d=1$ using the fact that the $\gamma-$fat shattering dimension of the set of functions with total variation smaller than $V$ is equal to $1 + \lfloor {V}/{2 \gamma} \rfloor$. We refer to Corollary 4.3 from \cite{simon1997bounds} for a detailed proof. Note that if we had considered functions on the circle, by effectively setting $f(0) = f(1)$, we would have had $\chi(\cdot) =\frac{1}{2} TV(\cdot)$, and therefore the claim for $d=1$ would be an equality.

In general, in the case where $2 \gamma > V$, a point can be shattered by constant functions; on the other hand, given two points $x,y$ that are shattered,
there must exist two real numbers $r_x,r_y$ and two functions $f,g$ in the family such that $f(x) > r_x + \gamma>r_x-\gamma >g(x)$ and $f(y) < r_y - \gamma < r_y + \gamma < g(y)$. Thus (depending if $r_x \geq r_y$ holds, or the opposite) either $f$ or $g$
must necessarily have a range of values larger than $2 \gamma$ and therefore its persistence must be larger than $2 \gamma$, which yields a contradiction.
Hence, in any dimension $\mathrm{fat}_\gamma (\mathcal{H}_V) = 1$ if
$2\gamma >V$.

Assume now $d=2$ and $2 \gamma \leq V$. Consider a set of $n$ points $x_1, \ldots, x_n$ in $[0, 1]^2$ that form a regular $n-$gon. 
Let $E \subseteq \{1, \ldots, n\}$ be an arbitrary subset of indices. 
We consider a function $f$ such that 
\[\qquad \qquad \qquad  \left\{
    \begin{array}{ll}
        f(x) = -V/2 \text{ if } x \in \mathrm{Conv}(x_i)_{i \notin E}, \\
        f(x_i) = V/2 \text{ if } i \in E, 
    \end{array}
\right.
\]
and $f$ increases smoothly on $\mathrm{Conv}(x_i)_{i=1}^n \setminus \mathrm{Conv}(x_i)_{i \notin E}$ and if $x \notin \mathrm{Conv}(x_i)_{i=1}^n$, \\ $f(x)=f(\Pi_{\mathrm{Conv}(x_i)_{i=1}^n}(x))$ where $\Pi_\mathcal{C}(x)$ denotes the projection of $x$ onto a convex set $\mathcal{C}$. The function $f$ defined as such has a persistence of $V$ and the set $\mathcal{H}_V$ therefore $\gamma$-shatters this set of $n$ points. Similar examples can be built for $d >2$.
\end{proof}

This observation highlights a challenge to overcome when constructing penalties involving topological persistence. This serves as our main motivation for our proposed penalties, in order to restrict the size of the set of candidate functions based on eigenbasis expansions.

\subsection{Empirical eigenfunctions} 

For simple manifolds (a flat open space, a torus or a sphere for instance), computing the spectrum of the Laplace-Beltrami operator is analytically tractable. However, for general manifolds this is not possible. Moreover, in practical problems, the manifold itself may be unknown. To deal with this, we take the standard empirical approach and build an undirected graph on the vertex set $V=\{X_1, \ldots, X_n\}$, with weights $W_{ij}$ between the vertex $i$ and the vertex $j$ that are computed according to the ambient metric. In the following experimental study, we consider nearest neighbor and Gaussian-similarity based graphs. We denote by $L=D-W$ the unnormalized graph Laplacian matrix with degree matrix $D$ and weight matrix $W$. The degree matrix is a diagonal matrix simply defined as
\[\qquad \qquad D_{ii} = \sum_{j=1}^n W_{ij} \text{ and } D_{ij} = 0 \text{ if } i \neq j.
\]
The normalized Laplacian matrix is then given by $L^{\prime} = D^{-1/2} L D^{-1/2}$. Both the matrix $L$ and $L^\prime$ are symmetric positive semi-definite and therefore admit a basis of orthogonal eigenvectors. We will only focus on the normalized Laplacian since it provides slightly better convergence guarantees~\citep{von2008consistency}. The use of these eigenvectors is justified by the fact that they converge to the true eigenfunctions of the Laplace-Beltrami operator in various metrics as the number of points $n$ tends to infinity and the scaling parameter of the graph tends to $0$~\citep{koltchinskii1998, trillos2020error}. A few estimated eigenfunctions based on nearest-neighbor graph Laplacian are plotted in Figure \ref{lap_eig}, for points regularly sampled on the unit square folded into a torus. The nearest-neighbor graph is built thanks to the ambient metric of $\mathbb{R}^3$ (and not the metric on the torus). This 
is justified because the results of \citet{trillos2020error} ensure the convergence of the spectrum of the graph Laplacian built on the ambient metric. This is due to the fact that locally, the metric on the manifold resembles the ambient metric (see Proposition 2 of \citealp{trillos2020error}). In what follows, the $i-$th eigenvector of the Graph Laplacian matrix is denoted by $\hat{\Phi}_i$.

We also remark that while we previously defined the persistence for the true Laplace-Beltrami eigenfunctions, we can simply extend the definition for estimated graph-Laplacian eigenfunctions on $V$ by considering $\bar{f}_i := \sum_{j=1}^n \hat{\Phi}_i (X_j) \mathbb{1}_{V_j}$ where $V_j$ is the Voronoi cell centered on $X_j$. We also use the notation $\chi(\hat{\Phi}_i) = \chi (\bar{f}_i)$. Finally, we also mention that using the spectrum of the Laplacian of a graph is a broadly developed idea to perform various statistical learning tasks such as regression or clustering; see, for example, \cite{chun2016eigenvector}, \cite{irion2014hierarchical}, \cite{ ng2002spectral} and \cite{von2008consistency}.

\begin{figure}[t]
\begin{center}
\subfigure[EF 1]{\includegraphics[width=0.3\linewidth]{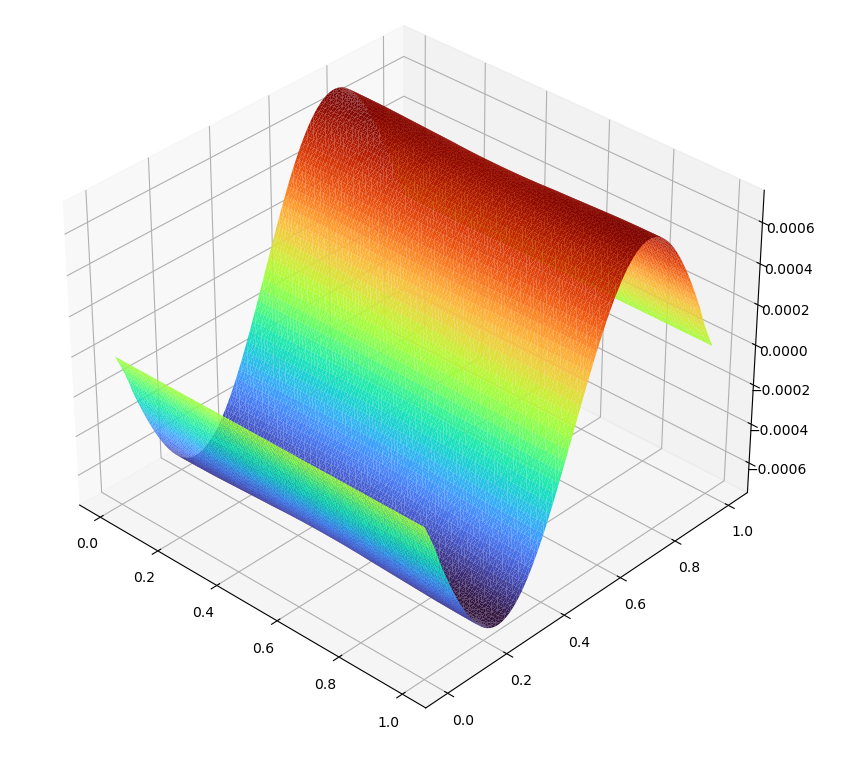}}
\subfigure[EF 2]{\includegraphics[width=0.3\linewidth]{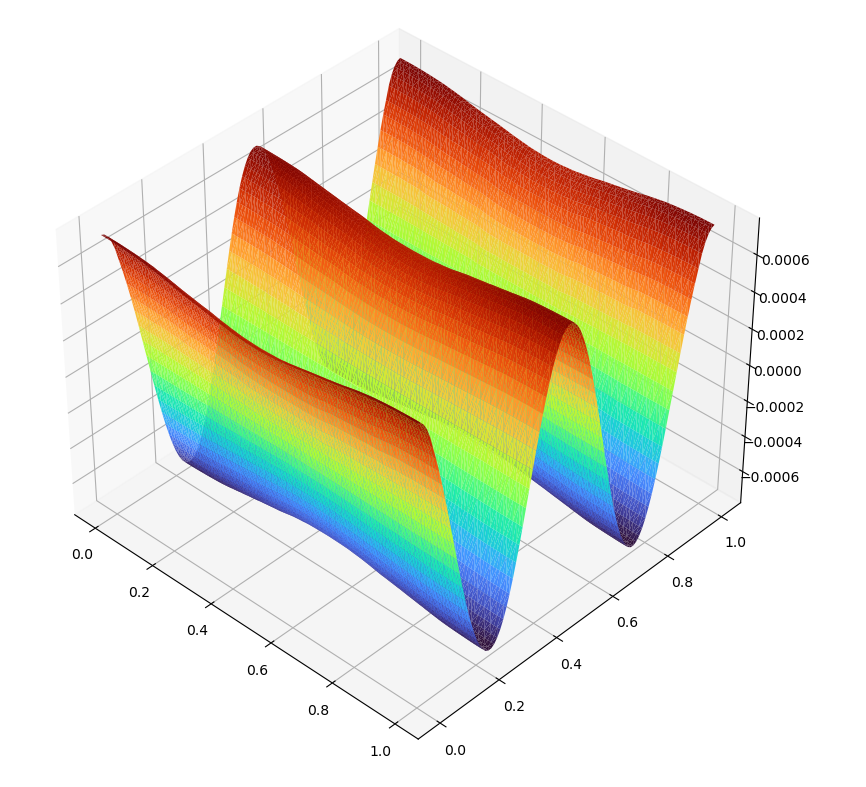}} 
\subfigure[EF 10]{\includegraphics[width=0.3\linewidth]{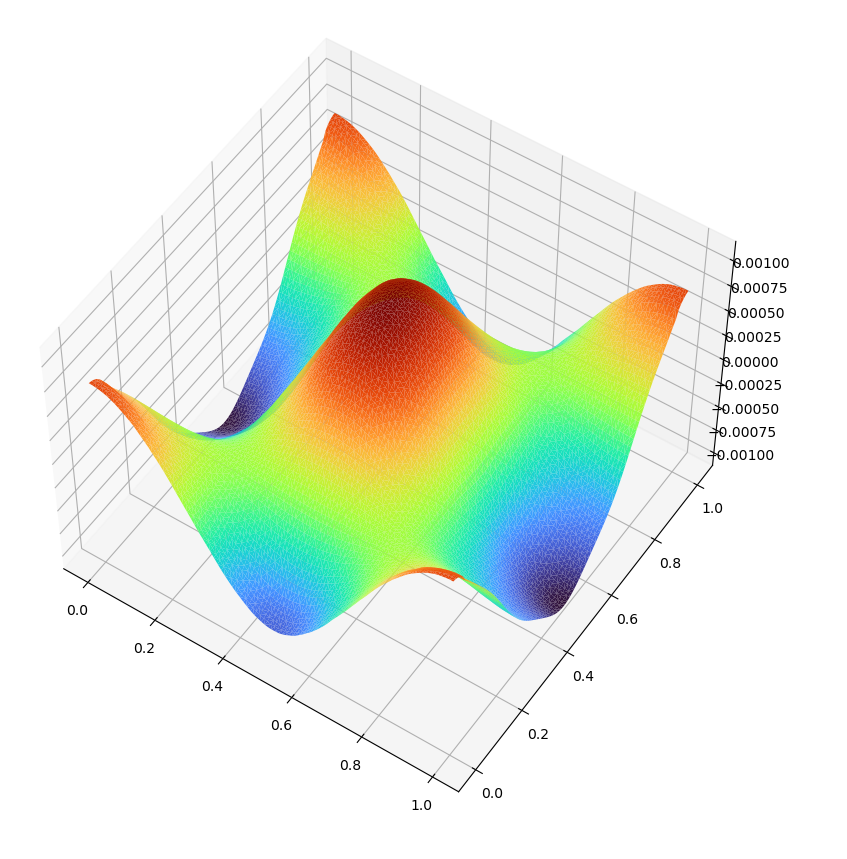}} \\
\subfigure[EF 27]{\includegraphics[width=0.3\linewidth]{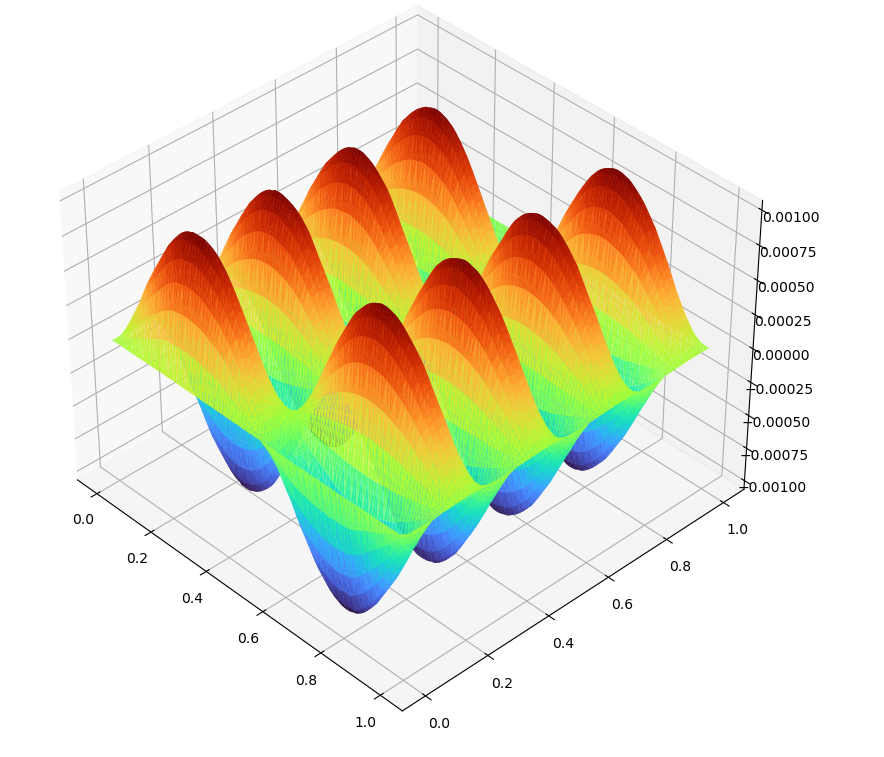}} 
\subfigure[EF 54]{\includegraphics[width=0.24\linewidth]{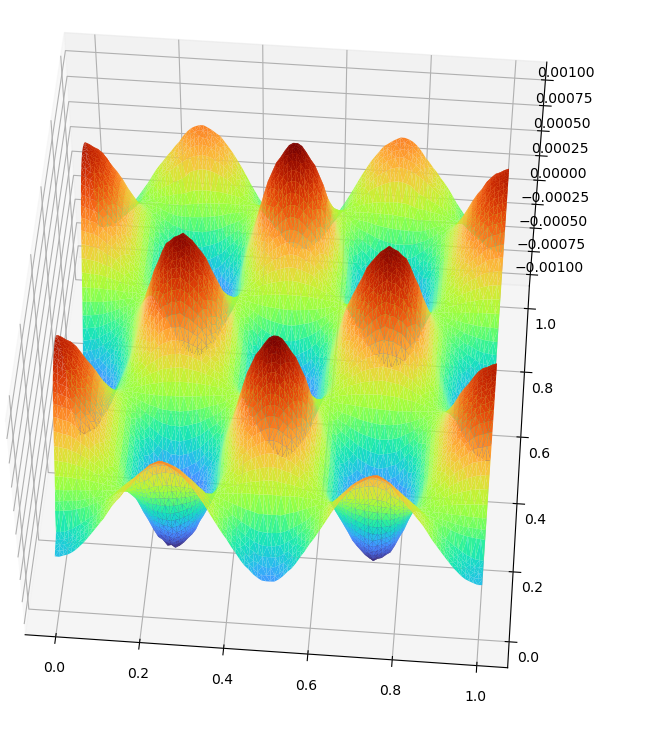}}
\subfigure[EF 75]{\includegraphics[width=0.3\linewidth]{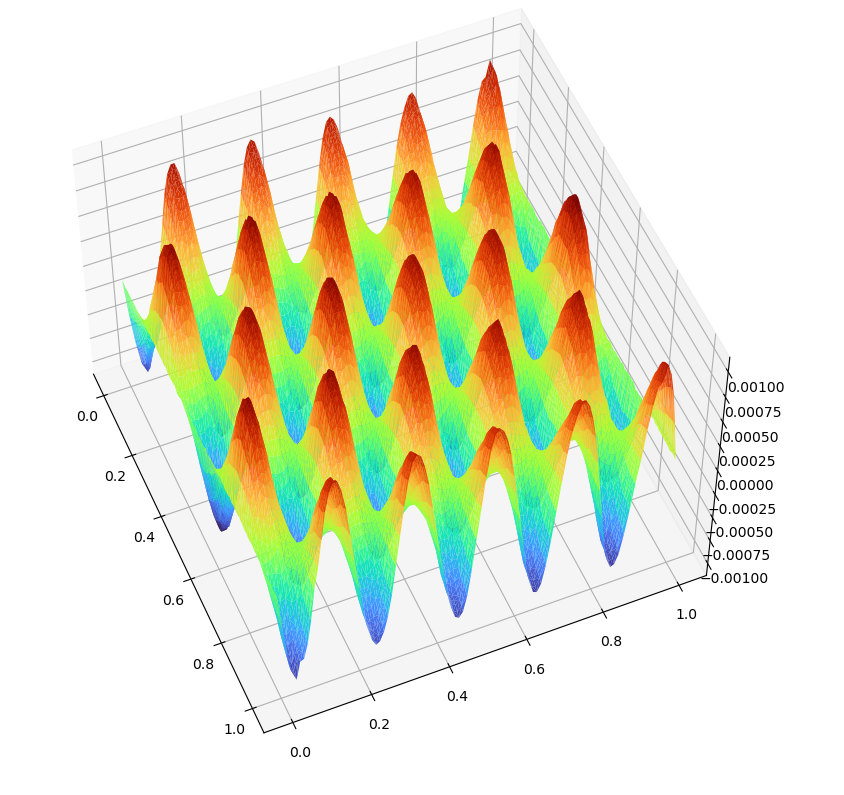}}
\end{center}
\caption{Several estimated eigenfunctions of the graph Laplacian, 10000 points sampled on $\mathbb{T}^2$,  8-NN graph.}
\label{lap_eig}
\end{figure}

\section{Theoretical guarantees}
\label{sec:theo}
We now provide novel theoretical guarantees for the proposed persistence regularization methodology.

\begin{assumption}
\label{asmp:onfstar}
We assume the true regression satisfies one of the two assumptions below.
\item[\textbf{A1}]  There exists $\theta^\star \in \mathbb{R}^p$ with $\|\theta^\star\|_0 = s$ such that $f^\star = \sum_{i=1}^p \theta_i^\star \Phi_i$, where $\Phi_j$ are the eigenfunctions of the Laplace-Beltrami operator with corresponding eigenvalue $\lambda_j$. In this case, we say that $f^\star$ has a sparsity index of $s$ over the basis $(\Phi_i)_{i=1}^p$. 
\item[\textbf{A2}]  There exists $\theta^* \in \mathbb{R}^p$ such that $f^\star = \sum_{j=1}^p \theta_j^\star \Phi_j$ and $f^\star$ is a Morse function.
\end{assumption}

A remark is in order regarding the sparsity assumption on $f^*$ in Assumption~\ref{asmp:onfstar}-\textbf{A1}. As discussed in Section~\ref{sec:tvvstp}, the relationship between the topological regularity of the eigenfunctions and their ordering is not known to be monotone in general, and it only exhibits a periodic trend. Hence, to maintain generality, we assume $f^*$ is a sparse linear combination of the eigenfunctions.

\subsection{Theoretical guarantees for the $\Omega_1$ penalization}
We are first interested in the properties of the penalty $\Omega_1$ introduced in Section \ref{two_types}. This approach can simply be understood as a weighted Lasso with a random design. Since the Laplace-Beltrami eigenfunctions form a basis of $L^2(\mathcal{M})$, the compatibility condition \citep{van2009conditions} is verified and we have a fast rate of convergence.

\begin{theorem}
\label{theo:lasso}
Assume that $f^*$ satisfies Assumption~\ref{asmp:onfstar}-\textbf{A1}. 
Assume we observe $Y_i = f^\star(X_i) + \varepsilon_i$ where $\varepsilon_i$ are zero-mean sub-Gaussian random variables with parameter $\sigma^2$.
Let $\hat{\theta}$ be the minimizer of $\mathcal{L}$
given by~\eqref{eq:penloss} with penalty $\Omega_1$ given by~\eqref{eq:omega1}. 
Then there exists a constant $C(\mathcal{M})$ that depends only on the manifold $\mathcal{M}$ such that for all $x>0$, we have that if
$$\mu \geq 2 \sigma \sqrt{\frac{pC(\mathcal{M})(\ln(p)+x)}{n}},$$ 
then with probability larger than $$1-2e^{-x}-e^{-\frac{0.15n}{C(\mathcal{M})p}+\ln(p)},$$ we have
\begin{align*}
\frac{1}{n}
\|\mathbf{X}(\hat{\theta}-\theta^\star)\|_2^2  + \mu \sum_{i=1}^p \chi(\Phi_i) |\hat{\theta}_i - \theta_i^\star| \leq \mu^2 \frac{s}{2}.
\end{align*}
\end{theorem}

The proof of the above theorem is provided in Section~\ref{sec:proofs}. For the result in Theorem~\ref{theo:lasso} to hold with large probability, the number of samples $n$ should be of order at least $p \ln(p)$. Under those circumstances, if the trade-off parameter $\mu$ is chosen of order $\sigma \sqrt{\frac{ p \ln(p)}{n}}$ as this theorem suggests, it can be shown that the overall prediction error is of order $ \frac{\sigma^2 p \ln(p)s}{n}$, up to multiplicative constants~\citep{van2009conditions}. According to Lemma~\ref{sup_norm} in Section~\ref{sec:proofs}, the use of a Laplace-Beltrami eigenbasis here translates into an additional multiplicative factor of $p$ as opposed to a standard Lasso with a design matrix satisfying the RIP conditions~\citep{buhlmann2011statistics}.

In the case where we study an approximation of the Laplace-Beltrami eigenbasis by using the estimated eigenfunctions $\hat{\Phi}_i$ of the graph Laplacian, the design matrix can be chosen to be orthonormal. Under those circumstances, the Lasso has an explicit solution, which we illustrate next. Let a design matrix $\mathbf{\hat{X}}$ be built on the estimated eigenfunctions or the graph Laplacian eigenvectors. Then, the minimizer $\hat{\theta}$ of the functional 
\begin{align}\label{eq:empobj}
\mathcal{L}^\prime(\theta) = \|Y-\mathbf{\hat{X}} \theta \|_2^2 + \mu \sum_{i=1}^p |\theta_i| \chi (\hat{\Phi}_i),
\end{align}
if a soft-thresholding type estimator given for all $j$ by:
\begin{align}\label{eq:thetaj}
 \hat{\theta}_j = \chi (\hat{\Phi}_j) \hat{\Phi}_j^T Y \bigg(1-\frac{\mu \chi (\hat{\Phi}_j)}{2 |\hat{\Phi}_j^T Y|} \bigg)_{+}.
\end{align}

To see this, consider a new design matrix $\mathbf{\hat{X}}$ such that  $\mathbf{\hat{X}}_{i,j} = \hat{\Phi}_j(X_i) / \chi (\hat{\Phi}_j)$. Minimizing the functional in~\eqref{eq:empobj} is then equivalent to minimizing the functional
\[ \qquad\qquad \mathcal{L}^{\prime}(\theta) = \|Y-\mathbf{\hat{X}} \theta\|^2+\mu \|\theta \|_1,
\]
by solving a standard Lasso problem. This function is convex in $\theta$ and therefore admits a global minimum (not necessarily unique) that we will denote by $\hat{\theta}$. Although $\mathcal{L}^\prime$ is not differentiable because of the $L^1$ penalty, we can still write the optimality conditions in terms of its sub-differential. Specifically, we have the sub-differential to be
\[ \qquad\partial \mathcal{L}^\prime (\theta) = \{ - 2\mathbf{\hat{X}}^T (Y-\mathbf{\hat{X}} \theta) + \mu z : z \in \partial \| \theta \|_1 \},
\]
from which the optimality condition could be obtained as 
\[\mathbf{\hat{X}}^T \mathbf{\hat{X}} \hat{\theta} = \mathbf{\hat{X}}^T Y - \frac{\mu}{2}\hat{z},
\]
where $\hat{z} \in \mathbb{R}^p$ is such that $\hat{z}_j = \mathrm{sgn}(\hat{\theta}_j)$ whenever $\hat{\theta}_j \neq 0$ and $\hat{z}_j \in [-1, 1]$ whenever $\hat{\theta}_j = 0.$ Due to the orthogonal design, the term $\mathbf{\hat{X}}^T \mathbf{\hat{X}}$ simplifies and a straightforward analysis of the possible cases given the sign or the nullity of $\hat{\theta}_j$, for all $j$ leads to the solution of $\theta_j$ given in~\eqref{eq:thetaj} for all $j$. We therefore have an explicit condition on the eigenbasis selection process, that is $\hat{\theta}_j = 0$ if and only if $ |\langle \hat{\Phi}_j, Y \rangle| \leq \frac{\mu}{2} \chi (\hat{\Phi}_j)$. Stated otherwise, an eigenvector with a high persistence has to explain the data significantly well to be kept in the model.

\subsection{Theoretical guarantees for the $\Omega_2$ penalization}
The penalty $\Omega_2$ being non-convex, a thorough theoretical study seems more complicated. Nonetheless guarantees on the prediction error and on the persistence of the reconstruction can be derived, as described below.

\begin{theorem}
\label{oracle_theo}
Let $f^*$ satisfy Assumption~\ref{asmp:onfstar}-\textbf{A2}. Assume we observe $Y_i = f^\star(X_i) + \varepsilon_i$ where $\varepsilon_i$ are zero-mean sub-Gaussian random variables with parameter $\sigma^2$. We further assume that for all $i=1,\ldots,n$, $X_i$ is sampled uniformly from $\mathcal{M}$ and that $\varepsilon_i$ is
independent of $X_i$. Let $\hat{\theta}$ be the minimizer of $\mathcal{L}$
given by~\eqref{eq:penloss} with penalty $\Omega_2$ given by~\eqref{eq:omega2}. Then for all $x>0$, the estimated parameter $\hat{\theta}$ verifies with probability larger than  
$$1 - 2e^{-x}-\exp \left(\frac{-0.1n}{C(\mathcal{M})p} + \ln (2p) \right),$$
that
\begin{align*}
\|\theta^\star - \hat{\theta} \|^2 \leq 16 \frac{p\sigma^2}{n} \left[1+ C(\mathcal{M}) \sqrt{\frac{2x}{n}} \right] (1+ \sqrt{x})^2  + 4 C(\mathcal{M}) p\, (2\nu(f^\star)+\zeta)^2 \mu^2.
\end{align*}
Here, we recall that $\mu$ is the trade-off parameter, $C(\mathcal{M})$ is a constant that depends only on the manifold $\mathcal{M}$, $\zeta$ is the total Betti number of $\mathcal{M}$ and $\nu(f^\star)$ is the number of points in the persistence diagram of $f^\star$. In addition, we also have under the same hypotheses with $\hat f = \sum_{j=1}^p \hat \theta_i \Phi_i$:
\[ \chi (\hat{f}) \leq \chi (f^{\star}) + 16 \frac{p \sigma^2}{\mu n}\left[1+ C(\mathcal{M})
\sqrt{\frac{2x}{n}} \right] (1+\sqrt{x})^2 + 8 C(\mathcal{M}) p (2\nu(f^\star)+\zeta)^2 \mu.
\]
\end{theorem}

The proof of the above theorem is provided in Section~\ref{sec:proofs}. This result holds with large probability 
if the number of samples $n$ is at least of order $O(p \ln(p))$. Choosing $\mu = O(1/\sqrt{n})$ ensures
that the trade-off term is of the same order as the main term, and we obtain a rate of convergence of order $O(p/n)$ for $\|\hat \theta - \theta^*\|$, which is what we can expect from such a model without any sparsity assumption. The second part of this theorem ensures the topological consistency of the reconstructed function $\hat{f}$, namely that it has a persistence that remains close to the persistence of the regression function $f^\star$ and is therefore topologically smooth to some extent. Again choosing $\mu = O(1/\sqrt{n})$ 
(as suggested above to keep a classical convergence rate for the parameter) leads to a consistency of the persistence of $\hat{f}$ towards that of $f^*$ at a rate $O(p/\sqrt{n})$. We therefore need a larger sample size (of order at least $O(p^2)$) to obtain consistency of the total persistence.

Note that according to Equation 6.14 from \cite{polterovich2019topological}, the number of features with persistence larger than some given value $c$ can be upper bounded by $(\kappa \|\nabla f\|_\infty/c)^d$ where $\kappa$ is a constant that depends only on the metric of the manifold. This shows a connection between the topological smoothness developed in this paper and a more usual notion of smoothness.

\subsection{Theoretical prospects}

A first possible extension of the theoretical results presented in this section can be to generalize Theorem \ref{oracle_theo} to mis-specified models, that is, cases where the target function $f^\star$ no longer belongs to $\mathrm{Span} (\Phi_1, \ldots, \Phi_p)$ but can be any $L^2$ function. To this end, a lower bound on the bias incurred with such a model for a function with fixed total persistence may be found in Proposition 2.1.1 from \cite{polterovich2019persistence} in the case of surfaces.

\begin{theorem}
\label{approx_neg}
Let $\mathcal{M}$ be a compact orientable Riemannian surface without boundary. Denote by $\mathcal{F}_\lambda$ the set of smooth functions over $\mathcal{M}$, such that for every $f \in \mathcal{F}_\lambda$, $\|f\|_2 = 1$ and $\| \Delta f \|_2 \leq \lambda$. Then there exists a constant $\kappa$ that only depends on $\mathcal{M}$ and its metric such that for every Morse function $f : \mathcal{M} \to \mathbb{R}$,

\[ \inf \{ \|f-h\|_\infty | h \in \mathcal{F}_\lambda \} \geq \frac{1}{2(\nu (f) +1)} \left( \chi(f) - \kappa (\lambda +1 ) \right).
\]

\end{theorem}

This means that for a fixed $\lambda$, if the persistence of the target function $f$ is too large, it will be impossible to approximate it with eigenfunctions of corresponding eigenvalue smaller than $\lambda$, and in order to have a chance of approximating it, we will have to allow for more oscillating functions by letting $\lambda$ increase. Balancing the estimation and approximation errors, based on the above result might lead to a data-driven choice for selecting $p$.

The other possibility of extension would be to establish consistency results for the case of estimated eigenfunctions, based on the graph Laplacian approach. For example, in order to derive an oracle inequality similar to that of Theorem \ref{oracle_theo}, we would need to establish the convergence of the persistence of the eigenvectors towards the persistence of the eigenfunctions. A potential approach is to leverage the stability results for persistence (for example, Lemma~\ref{stability}), and combine them with error rates for eigenfunction estimation (for example, recent results by~\citealp{dunson2021spectral} and \citealp{calder2020lipschitz}  for the empirical uniform norm). However, there are several technical challenges to overcome, in order to implement the above proof strategy. For example, it is not clear if the estimated eigenfunctions satisfy the regularity conditions required, for example, by stability results like Lemma~\ref{stability}. Furthermore, the results from~\cite{dunson2021spectral} and \cite{calder2020lipschitz} are existence results, and do not resolve the inherent identifiability issue arising in eigenfunction estimation. In addition, the following two cases are to be distinguished:
\begin{itemize}
    \item Semi-supervised setting: This corresponds to the case when there is an additional set of unlabeled observations $(X_{n+1}, \ldots, X_{n+m})$ uniformly and independently sampled over $\mathcal{M}$. In this case, the eigenfunctions could first be estimated using the above unlabeled observations and the regression coefficients could be subsequently estimated using the estimated eigenfunctions. 
    \item Supervised setting: In this case, the same set of observations $(X_1, \ldots, X_n)$ is used to estimate the eigenfunctions and the regression coefficients. Extra difficulties arise in this case due to the dependency in estimating the eigenfunctions and the regression coefficients using the same set of observations. We remark that this is the approach we take in the experiments as we discuss in Section \ref{sec:expe}.
\end{itemize}

\section{Experimental results}
\label{sec:expe}
\subsection{Experimental design}
\label{sec:expdesign}

We have applied the following experimental routine in order to estimate the function $f^\star$, given a set of points $(X_i)_{i=1}^n$ in $\mathbb{R}^D$ that are assumed to lie on a manifold $\mathcal{M}$ of dimension $d$ and a vector of real responses $(Y_i)_{i=1}^n.$
\begin{itemize}
\item Build a proximity graph on the data points $X_i$. Many options are possible, in most experiments we have taken a $k-$nearest neighbor graph with $k \simeq \log(n)$ but we also sometimes consider Gaussian weighted graphs, for instance in Section \ref{cat1D}.
\item Compute the normalized Laplacian matrix of the graph.
\item For a fixed $p \leq n$, compute the first $p$ eigenvectors of the Laplacian matrix, which yields a new design matrix in $ \mathbb{R}^{n \times p}$ where each column is an eigenvector.
\item For the penalty $\Omega_1$, compute the persistence of each eigenvector, divide each column of the design matrix by its persistence and solve a Lasso with cross-validation.
\item For the penalty $\Omega_2$, we start from a random vector $\theta_0 \in \mathbb{R}^p$ and perform a stochastic gradient descent of $\mathcal{L}$, where we compute the persistence of $\sum_{i=1}^p \theta_i \hat{\Phi}_i$ at each iteration, similarly to what is done in \cite{carriere2020note}.
\end{itemize}

The method that has been the most efficient is to take $p=n$ for $\Omega_1$ to perform a variable selection, using Lasso sparsity properties. We then perform a gradient descent of the loss function with penalty $\Omega_2$ on the subset of eigenvectors previously selected. The "vanilla" penalty $\Omega_2$ (without pre-selection step) is itself numerically outperformed in terms of MSE by $\Omega_1$. This can be explained by Theorem \ref{oracle_theo} : without performing a preliminary variable selection, the dimension of the problem does not guarantee a good reconstruction of the source function. This dimension reduction also offers a modest acceleration of the computational cost for the optimization of the loss with penalty $\Omega_2$ as we will see in Section \ref{complexity}. In the tables below, the results for the penalty $\Omega_2$ are always understood to have been obtained this way.

The routine previously described works for the denoising problem where we have a label for each data point. If we are interested in prediction problems, where the set of covariates is split in two: one part with labels and one part without, we build the graph on all the data points since all the points are assumed to lie on $\mathcal{M}$ but we train the model only on the points for which we have a label at our disposal.

Numerically, to compute the persistence of a function $f$ for which we know the values at points $x_1, \ldots, x_n$, we build the alpha-complex introduced in \citep{edelsbrunner2010computational} on the vertex set $(x_1, \ldots, x_n)$ where the filtration value of a $k$-dimensional simplex $[x_{i_0}, \ldots, x_{i_k}]$ is equal to $\underset{i \in \{ i_0, \ldots, i_k \} }{\max} f(x_i)$. For the alpha-complex to convey the same topology as the underlying manifold, we truncate it by keeping only simplices whose circumradius is small enough (namely smaller than half the reach of the manifold, see e.g., \citealp{boissonnat2018geometric}). This reach parameter is known in the simulations of Section \ref{simdata}. Should it not be known, it could be estimated (for instance, using methods developed by \citealp{berenfeld2022estimating}). In the real data examples of Section \ref{realdata}, the underlying simplicial complex is the extension of the nearest-neighbor graph used to compute the Laplacian eigenbasis. We claim that if the number of neighbors is chosen with care, we will also retrieve the topology of the underlying manifold. If using Gaussian-weighted graphs, we introduce a proximity parameter which is equivalent to considering a truncated Rips-filtration. Note that this simplicial complex is only computed once, and does not impact the overall computational cost, even when the ambient dimension is high.  Here, we have used the GUDHI library~\citep{maria2014gudhi} to compute the persistence given this filtration.

In what follows, we will present the results of several experiments conducted both on synthetic and real data to investigate the relevance of our approach in practice. We compare our method to standard regression methods on manifolds: Kernel Ridge Regression (KRR), Nearest-neighbour regression (k-NN), Total variation penalty (TV) as well as graph Laplacian eigenmaps with a $L^{1}$ penalty (Lasso) and a weighted Lasso where the weights are the total variation of each eigenvector, computed on the graph (Lasso-TV). The performance is measured in terms of root mean squared error between the estimated and the true functions at the data points. Its expression is given by 
\[\qquad \qquad RMSE(\hat{f}) = \sqrt{\sum_{i=1}^n \frac{(f^{\,^{\boldsymbol \star} }(x_i) - \hat{f} (x_i))^2}{n}}.
\] All hyperparameters have been tuned by cross validation or grid-search. The code used for data on a Swiss roll and the spinning cat in 2D is available here.\footnote{ \url{https://github.com/OlympioH/Lap_reg_topo_pen}}

\subsection{Simulated data}
\label{simdata}
\subsubsection{Illustrative example}
In order to illustrate the behavior of the $\Omega_2$-penalization model, we first look at a synthetic setting where the function we try to reconstruct is a sum of 4 Gaussians, to which we add a large noise (Figure \ref{gaussian_data}).

The persistence diagram of the true function to be estimated has four points in its 1-homology corresponding to the 4 cycles in each Gaussian, all being born at a neighboring saddle and dying at the corresponding local maximum, and one point in its 0-homology dying at infinity corresponding to the only connected component. The 0-homology points on the diagonal correspond to sampling noise. When noise is added, the visual chaos of the observation is supported by its persistence diagram: it has a lot of features and the statistical noise added to the measurement here is converted into topological noise. 

In Figure \ref{result_gaussians} we compare the results of a stochastic gradient descent penalizing $\Omega_2$ against a Lasso estimation. The reconstruction is visually better and is also better in terms of MSE. Indeed, when penalizing the persistence, we have managed to keep the four most persistent one-dimensional features in the persistent diagram (although their persistence has diminished) and we still observe four peaks, while most of the noise has been removed. Although the Lasso enables some denoising, it has only been able to reconstruct 2 to 3 peaks and does not offer the same denoising performance. Note that it is also possible to consider a topological penalty where we penalize the persistence of the function except the 4 most persistent points in 1-homology and the most persistent point in 0-homology like in \cite{bruel2019topology}. In this case, we can be more coarse in the choice of the trade-off $\mu$ since the penalty $\Omega_2$ must then be set to 0. This shows that if one is given an a priori on the topology of the function to reconstruct, it can be included in the model very easily.

\begin{figure*}[h]
\begin{center}
\subfigure[Function estimated by a Lasso]{\includegraphics[scale=0.3]{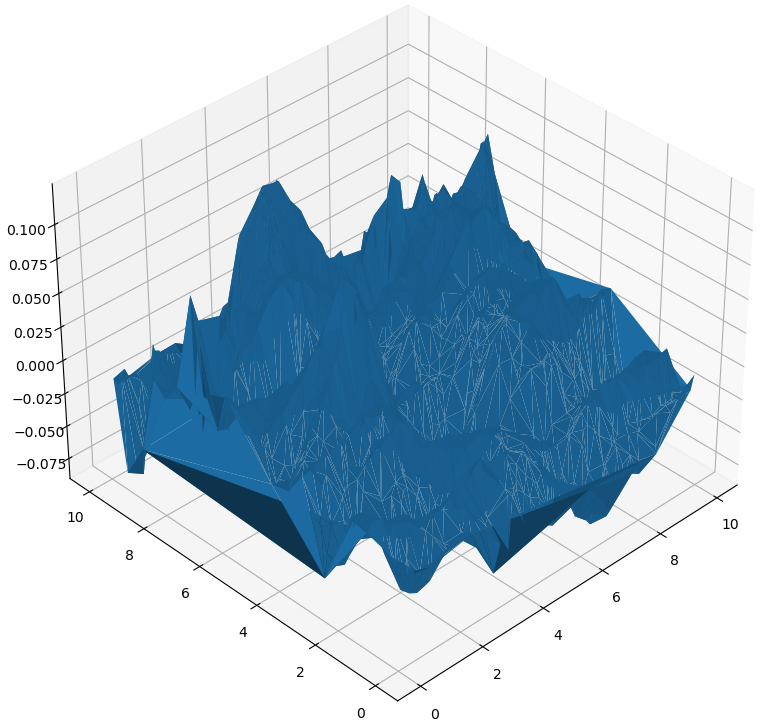}}
\subfigure[Function estimated by the topological penalty $\Omega_2$]{\includegraphics[scale=0.3]{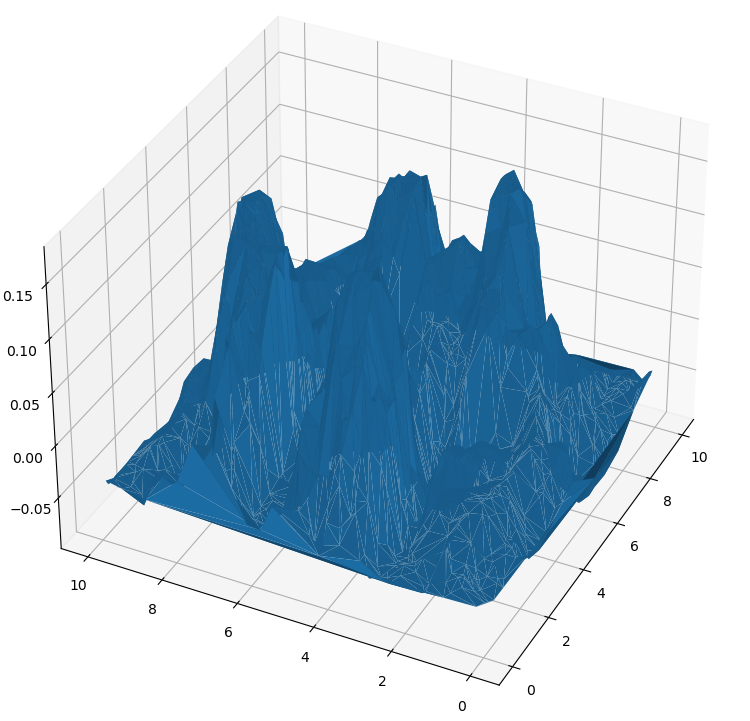}}\\
\subfigure[Persistence diagram of the Lasso estimate]{\includegraphics[width=0.4\linewidth]{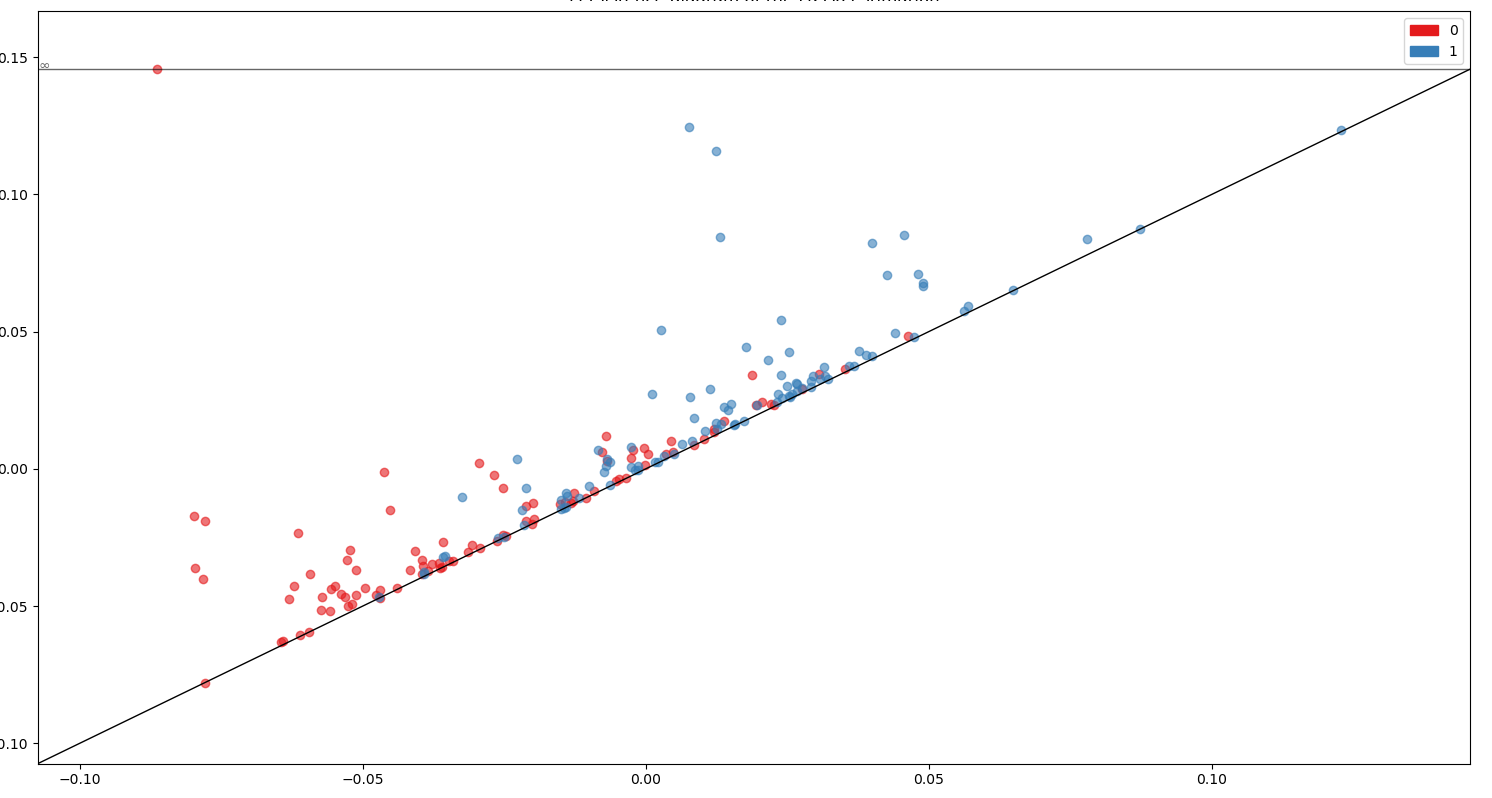}} 
\subfigure[Persistence diagram of the $\Omega_2$ penalty estimate]{\includegraphics[width=0.4\linewidth]{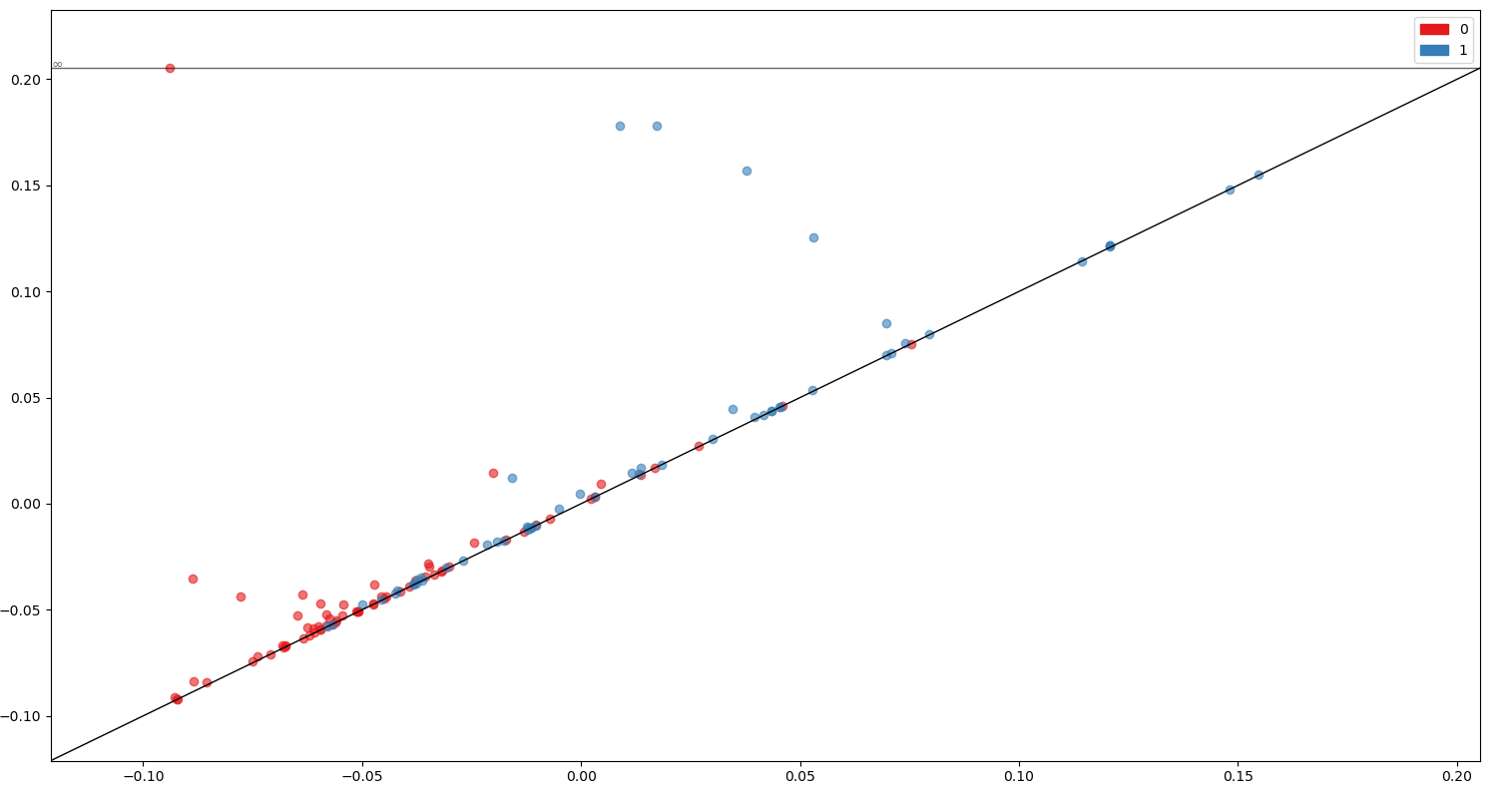}}
\caption{Reconstruction of the sum of four Gaussians.}
\label{result_gaussians}
\end{center}
\end{figure*}

\subsubsection{Torus data}
\label{sec:torus}

We simulate data on a torus which we recall is homeomorphic to $\mathbb{S}^1 \times \mathbb{S}^1$, parametrized by the embedding $\Psi_{\mathbb{T}^2} : (\theta, \phi) \mapsto (x_1, x_2, x_3)$:

\[ \qquad \qquad \left\{
    \begin{array}{lll}
        x_1 = (2+\cos(\theta))\cos(\phi), \\
        x_2 = (2+\cos(\theta))\sin(\phi), \\
        x_3 = \sin (\theta).
    \end{array}
\right.
\]

Following the approach of \cite{diaconis2013sampling}, sampling uniformly on the torus may be carried out via sampling $\phi$ uniformly in $[0, 2\pi]$ and $\theta$ according to the density \\  $g(\theta)=\frac{1}{2 \pi} \left( 1+\frac{\cos(\theta)}{2}\right)$ on $[0, 2\pi]$. In practice, this is performed by a rejection sampling. Note that sampling $\theta$ and $\phi$ uniformly and independently does not provide a uniform sampling on the torus as we will observe a higher density of points in the inside of the torus. A function on the torus is identified with a function of $(\theta, \phi)$. To set up the regression problem, we define the target response $f^*(\theta,\phi)=\xi [-17(\sqrt{(\theta - \pi)^2+(\phi - \pi)^2} - 0.6 \pi)]$ where $\xi$ is the sigmoid function. Note that $f^*$ is radial symmetric, depending only on the distance between $(\theta, \phi)$ and $(\pi, \pi)$. This signal function has been studied because it has a simple topology that illustrates the topological denoising method. It has first been introduced by \citet{nilsson2007regression} for similar purposes. The comparison of the RMSE for all methods can be found in Table \ref{tab:torus}.

\begin{table*}

\begin{center}
\caption{RMSE of the reconstruction for $n$ points lying on a torus, average on $100$ runs.}

\resizebox{\textwidth}{!}{
\begin{tabular}{|c|c|c|c|c|c|c|c|c|}

	\hline
$n$ & $\sigma$ & Lasso & Lasso-TV  & $\Omega_1$ & $\Omega_2$ & KRR & k-NN & TV \\
	\hline
300 &	0.5 & 0.261 $\pm$ 0.019  & 0.241 $\pm$ 0.017 & 0.220 $\pm$ 0.019 & 0.223 $\pm$ 0.019 & \bf{0.217} $\pm$ 0.012 & 0.237 $\pm$ 0.017 & 0.413 $\pm$ 0.016 \\
	\hline
300 &	1 & 0.381 $\pm$ 0.038  & 0.337 $\pm$ 0.043 & \bf{0.281} $\pm$ 0.038 & 0.288 $\pm$ 0.037 & 0.292 $\pm$ 0.025 & 0.401 $\pm$ 0.029 & 0.509 $\pm$ 0.026\\
	\hline	
	1000 & 0.5&	0.174 $\pm$ 0.011 & 0.171 $\pm$ 0.010 & 0.157 $\pm$ 0.010 & \bf{0.156} $\pm$ 0.011 & 0.172 $\pm$ 0.008 & 0.167 $\pm$ 0.009 & 0.421 $\pm$ 0.008 \\
		\hline	
	1000 & 1 &	0.290 $\pm$ 0.024 & 0.266 $\pm$ 0.021 & 0.212 $\pm$ 0.018 & \bf{0.209} $\pm$ 0.018 & 0.222 $\pm$ 0.017 & 0.285 $\pm$ 0.013 & 0.513 $\pm$ 0.015 \\
	\hline 

\end{tabular}}
\label{tab:torus}
\end{center}
\end{table*}

\subsubsection{Data on a Swiss roll}

We now consider data on a Swiss roll which is a two dimensional manifold parametrized by the mapping $(x, y) \mapsto (x \cos x, y, x \sin x)$. We have set as target function $$f^*(x, y) =4\exp(-((y-7)^2/20+(x-6)^2/5))+ 2 \cos^2 (x) \sin ^2 (y),$$ for $(x,y) \in [1.5 \pi, 3.5 \pi] \times [0, 21]$. The function is plotted Figure \ref{swiss_func}.

\begin{figure}[t]
\begin{center}
\subfigure[Original function]{\includegraphics[scale=0.45]{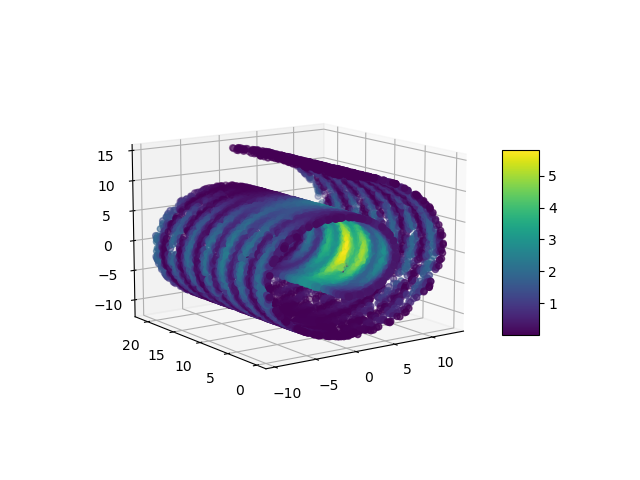}}
\subfigure[Noisy observation]{\includegraphics[scale=0.45]{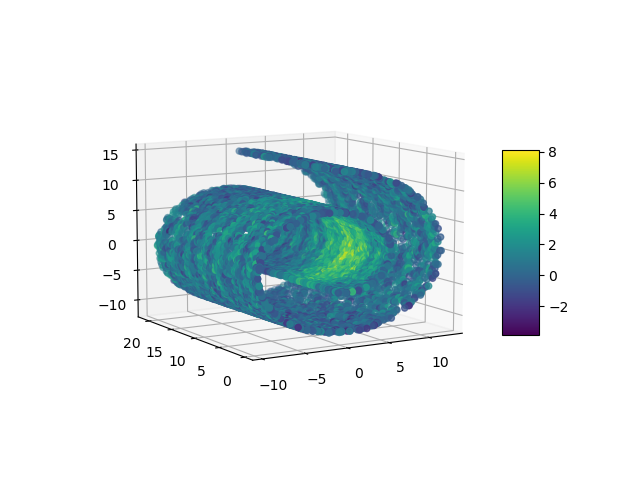}} 
\end{center}
\caption{Regression function on a Swiss roll}
\label{swiss_func}
\end{figure}

\begin{table*}
\begin{center}
\caption{RMSE of the reconstruction for 500 points on a Swiss roll, average on 100 runs.}

\resizebox{\textwidth}{!}{
\begin{tabular}{|c|c|c|c|c|c|c|c|}
\hline
$\sigma$ & Lasso & Lasso - TV  & $\Omega_1$ & $\Omega_2$ & KRR & k-NN & TV\\
	\hline
			0.2 & 0.422 $\pm$ 0.035 & 0.421 $\pm$ 0.031 & 0.491 $\pm$ 0.065 & 0.398 $\pm$ 0.026   & \bf{0.186}  $\pm$ 0.006 & 0.505 $\pm$ 0.020 & 0.200 $\pm$ 0.006 \\
				\hline	
	0.5  & 0.498 $\pm$ 0.036 & 0.478  $\pm$ 0.026  & 0.472 $\pm$ 0.043 & \bf{0.455} $\pm$ 0.020 & 0.476 $\pm$ 0.014  & 0.532 $\pm$ 0.020 & 0.494 $\pm$ 0.017 \\
	\hline	
	0.7  & 0.549 $\pm$ 0.041 & 0.525  $\pm$ 0.033  & 0.534 $\pm$ 0.032 & \bf{0.489} $\pm$ 0.021 & 0.526 $\pm$ 0.017  & 0.549 $\pm$ 0.021 & 0.671 $\pm$ 0.022 \\
	\hline
		1 & 0.628 $\pm$ 0.0430 & 0.593 $\pm$ 0.035 & 0.572 $\pm$ 0.0600 & \bf{0.546} $\pm$ 0.030 & 0.637 $\pm$ 0.026 & 0.587 $\pm$ 0.023 & 0.920 $\pm$ 0.029\\
		\hline
		1.3 & 0.689 $\pm$ 0.049 & 0.648 $\pm$ 0.042 & 0.615 $\pm$ 0.064 &  \bf{0.595} $\pm$ 0.038  & 0.746 $\pm$ 0.026 & 0.646 $\pm$ 0.025 & 1.056 $\pm$ 0.043 \\
		\hline		
\end{tabular}}
\label{tab:swiss}
\end{center}
\end{table*}

Here, we observe that the penalty $\Omega_2$ performs the best and benefits from the selection properties of the regularization with penalty $\Omega_1$. Data on a Swiss roll shows the limitations of Kernel methods when the number of points is low, as the geodesic metric can be very different from the ambient metric. The use of a $k-$NN graph is a solution to bypass this problem since at a small scale the two metrics are close. We see in Table \ref{tab:swiss} that a topology-based penalty is particularly efficient when the noise level becomes important. Although Kernel Ridge regression and Total Variation penalties provide a very good reconstruction and should be preferred when the noise is low, they become quite unstable as the noise level increases, whereas all Laplacian eigenmaps based methods as well as a k-NN regression are somehow robust to noise. Note that among all possible penalties on Laplacian eigenmaps based models, $\Omega_2$ ran on the eigenfunctions selected by $\Omega_1$ always yield the best results.

\subsection{Real data}
\label{realdata}
\subsubsection{Spinning cat in 1D}
\label{cat1D}
This model has also been tried on real data. We consider a data set of $n=72$ images of the same object rotated in space with increments of $5 \degree$. The data lie on a one-dimensional submanifold of $\mathbb{R}^{16384}$, the images having size $128 \times 128$ pixels. We can see some of the images from the dataset Figure \ref{cat}.
For each image, we want to retrieve the angle of rotation in radians of the object. The source vector we want to estimate is therefore $(0, 5 \pi /180, 10 \pi /180, \ldots, 355 \pi /180) \in \mathbb{R}^{72}$ .

\begin{figure*}[t]
\begin{center}
\subfigure[Baseline image]{\includegraphics[scale=01]{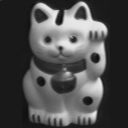}}
\subfigure[Image rotated by $30 \degree$]{\includegraphics[scale=1]{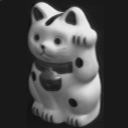}} 
\subfigure[Image rotated by $90 \degree$]{\includegraphics[scale=1]{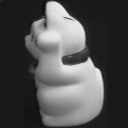}} 
\caption{Data set for the experiments on a one dimensional manifold with real data.}
\label{cat}
\end{center}
\end{figure*}

We have built a Gaussian weighted graph on the data points, using the ambient $L^2$ metric between images for the weights. Using a geodesic distance between images would probably yield better results, but the use of the ambient metric is enough to have convergence of the graph Laplacian eigenvectors towards the eigenfunctions of the Laplace-Beltrami operator on the manifold according to \cite{trillos2020error}.

We have tried different values of the scaling parameter $t$ in the Gaussian weights 
\[W_{ij}=\frac{1}{n} \frac{1}{t(4 \pi t)^{d / 2}} e^{-\frac{\left\|x_{i}-x_{j}\right\|^{2}}{4 t}}.
\]

We depict in Figure \ref{RotEF} the aspect of some eigenfunctions as a function of the rotation degree of the baseline image for $t=1$ and $t=10$. Here, we can see that the value $t=1$ is too small: indeed, although the graph-Laplacian eigenvectors converge to the true Laplace-Beltrami eigenfunctions as $n \to \infty$ and $t \to 0$, this only occurs if $t$ verifies a particular scaling with respect to $n$ according to \cite{trillos2020error}. Here, we only have a small number of data at hand ($n=72$), and therefore, the value $t=10$ visually seems to be more satisfying. Indeed, the data are on a circle (of some large dimensional Euclidean space) and we would expect the eigenfunctions to converge towards the spherical harmonics for the circle, which are oscillating functions. For $t=1$, all the eigenfunctions are highly localized which is not quite satisfying. In addition, a preliminary study yielded a much better performance of the parameter $t=10$ over $t=1$ for the corresponding regression task. For a larger index, the eigenfunctions start to be localized around an image of the manifold, reminding a wavelet-type basis.

\begin{figure*}[t]
\begin{center}
\subfigure[t=1, EF 1]{\includegraphics[width=0.30\linewidth]{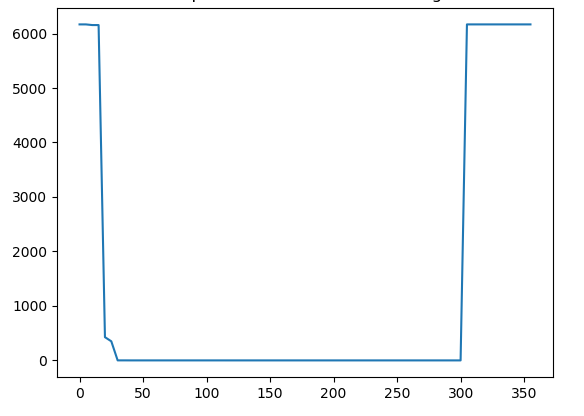}}
\subfigure[t=1, EF 13]{\includegraphics[width=0.30\linewidth]{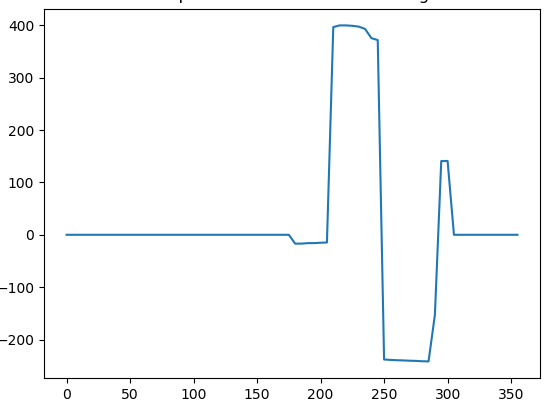}} 
\subfigure[t=1, EF 47]{\includegraphics[width=0.30\linewidth]{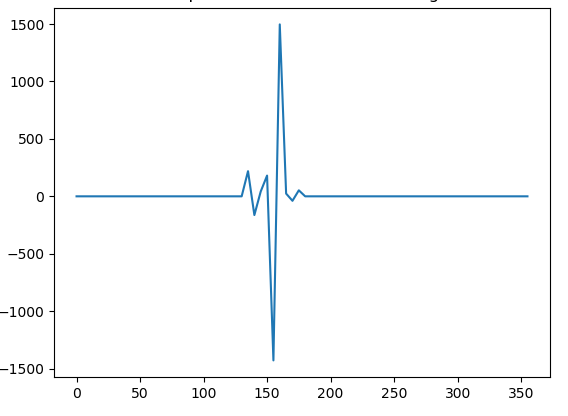}} \\
\subfigure[t=10, EF 1]{\includegraphics[width=0.30\linewidth]{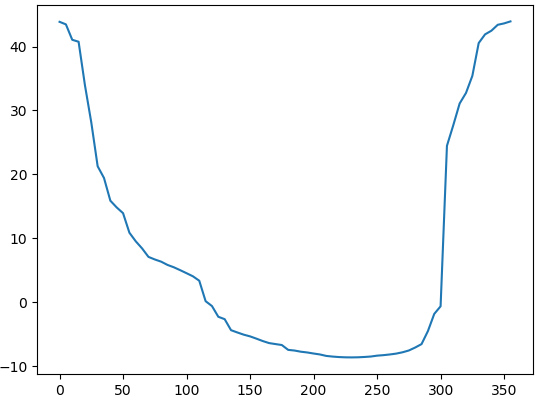}}
\subfigure[t=10, EF 13]{\includegraphics[width=0.30\linewidth]{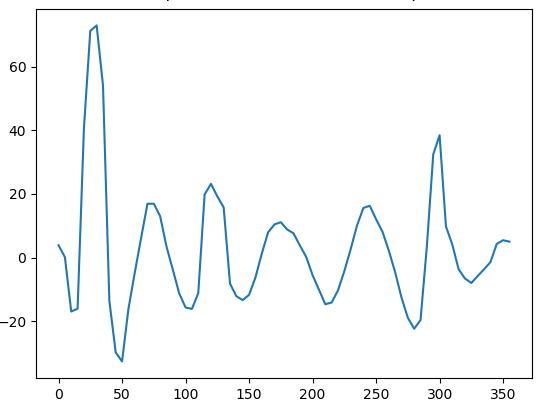}} 
\subfigure[t=10, EF 47]{\includegraphics[width=0.30\linewidth]{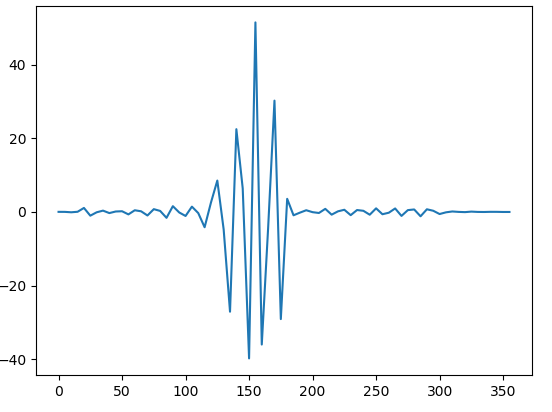}} \\
\caption{Estimated eigenfunctions of the manifold Laplacian evaluated at each image, as functions of the degree of rotation of the object in the image.}
\label{RotEF}
\end{center}
\end{figure*}

Unlike the synthetic experiments, where we have only focused on denoising, here we are interested in the prediction properties of the method. To this aim, the data are randomly split between a training set and a validation set. The graph is then built on the whole dataset (since the data points are all assumed to lie on the same manifold). The optimization procedure is then only performed on the labels from the training set, and we measure the mean square error between the prediction on the new points and the true values from the validation set. 

\begin{table*}
\begin{center}
\caption{RMSE of the prediction of the angle of rotation of the spinning cat in 1D, average on $100$ runs.}
\vspace*{0.3cm}

\resizebox{\textwidth}{!}{
\begin{tabular}{|c|c|c|c|c|c|c|c|}
	\hline
$\sigma$ & Size train/test & Lasso & Lasso-TV  & $\Omega_1$ & $\Omega_2$ & KRR & k-NN \\

	\hline	
	0 & 68/4  & 0.79 $\pm$ 0.59 & 0.56 $\pm$ 0.52 & 0.38 $\pm$ 0.51 & 0.35 $\pm$ 0.49 & \bf{0.22} $\pm$ 0.47 & 0.25 $\pm$ 0.50 \\
	\hline	
	0.5 & 68/4  & 0.89 $\pm$ 0.60  & 0.70 $\pm$ 0.51 & 0.59 $\pm$ 0.52 & 0.53 $\pm$ 0.48 & \bf{0.47} $\pm$ 0.47 & 0.44 $\pm$ 0.47 \\
	\hline
	0 & 60/12  & 0.96 $\pm$ 0.50 & 0.74 $\pm$ 0.46 & 0.58 $\pm$ 0.47 & 0.54 $\pm$ 0.46 & \bf{0.45} $\pm$ 0.41 & 0.65 $\pm$ 0.41 \\
	\hline
	0.5 & 60/12  & 1.04 $\pm$ 0.56 & 0.85 $\pm$ 0.50 & 0.71 $\pm$ 0.48 & 0.66 $\pm$ 0.49 & \bf{0.57} $\pm$ 0.42 & 0.77 $\pm$ 0.44 \\
 \hline
\end{tabular}}
\end{center}
\end{table*}	
We see here that although an approach based on the penalization of the topological persistence yields results a lot better than a $L^1$ or total variation penalty, it is still outperformed by a kernel ridge regression. This might be due to the small number of data available on which to build the graph or the fact that the topology of the manifold as well as the topology of the source function are very simple and do not benefit fully from the method presented in this paper. We will see in the next subsection a set-up that shows the appeal of such a penalty.

\subsubsection{Spinning cat in 2D}

We consider the data set from the previous subsection where in addition each image is rotated in another direction by increments of $5 \degree$. We thus dispose of a two-dimensional manifold with the homotopy type of a torus, where the two parameters are the degree of rotation of the object within the image $\theta$ and the degree of rotation of the image itself $\phi$. A few images of the dataset are plotted in Figure \ref{2D_dataset}. We consider the same target response as in the Subsection \ref{sec:torus}: 
\[
f^*(\theta,\phi) = \xi [-17(\sqrt{(\theta-\pi)^2+(\phi-\pi)^2} - 0.6 \pi)].
\]
We add an i.i.d. Gaussian noise with standard deviation $\sigma = 1$ to the input. We select a random subset of 1000 images over the dataset and randomly split it into a training set and a testing set. In this example, we only penalize the $0-$persistence since it has provided a better performance on a preliminary study. The results can be found Table \ref{table_cat2D}.

\begin{figure*}[t]
\begin{center}
\subfigure[$\theta = 0 \degree, \phi = 0 \degree$]{\includegraphics[width=0.3\linewidth]{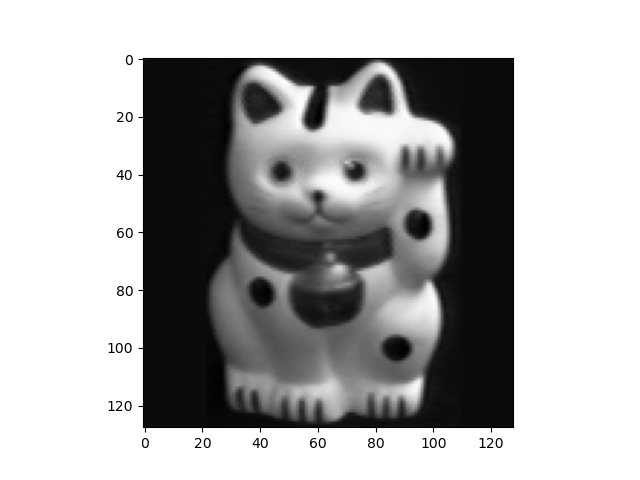}}
\subfigure[$\theta = 30 \degree, \phi = 0 \degree$]{\includegraphics[width=0.3\linewidth]{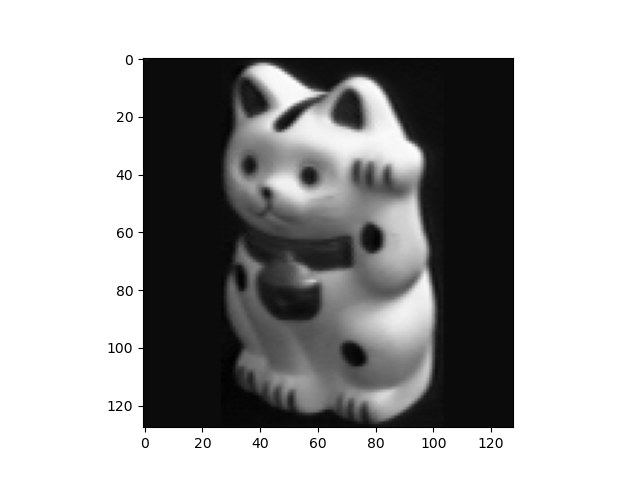}} 
\subfigure[$\theta = 120 \degree, \phi =0 \degree$]{\includegraphics[width=0.3\linewidth]{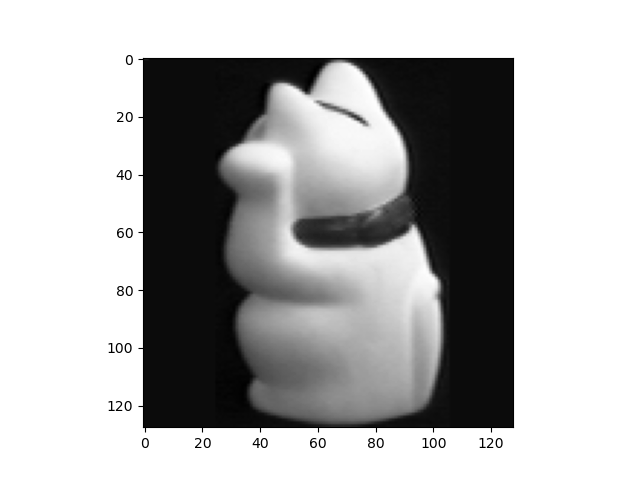}} \\
\subfigure[$\theta = 0 \degree, \phi = 30 \degree$]{\includegraphics[width=0.3\linewidth]{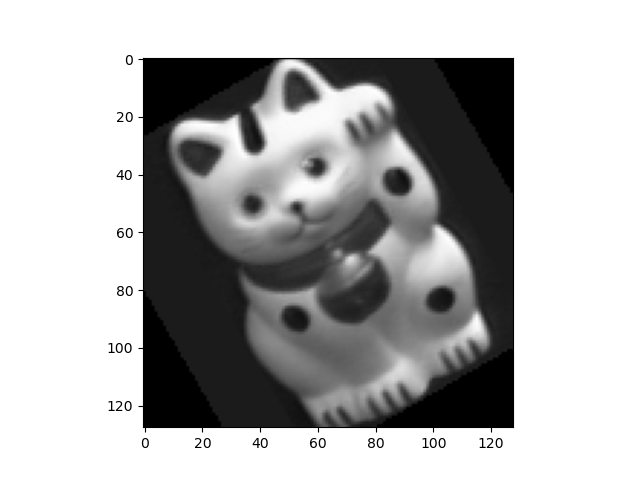}}
\subfigure[$\theta = 30 \degree, \phi = 30 \degree$]{\includegraphics[width=0.3\linewidth]{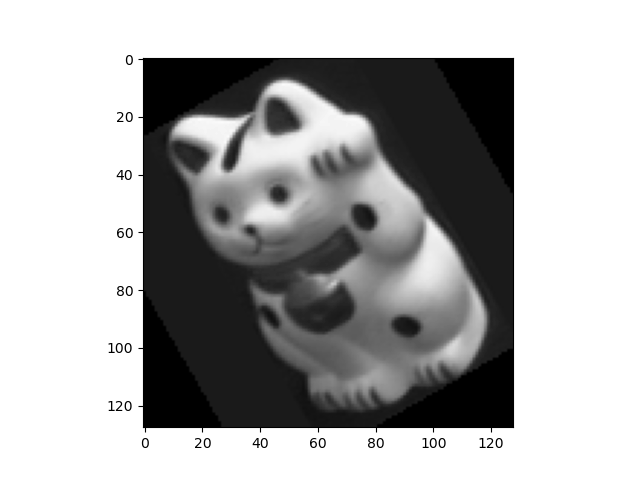}} 
\subfigure[$\theta = 120 \degree, \phi =30 \degree$]{\includegraphics[width=0.3\linewidth]{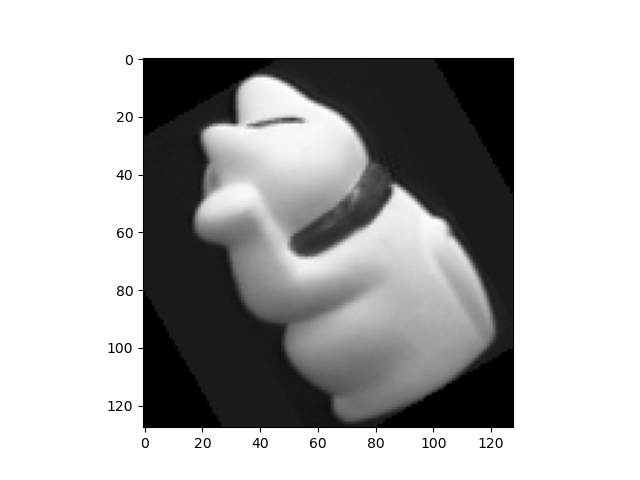}} \\
\subfigure[$\theta = 0 \degree, \phi = 90 \degree$]{\includegraphics[width=0.3\linewidth]{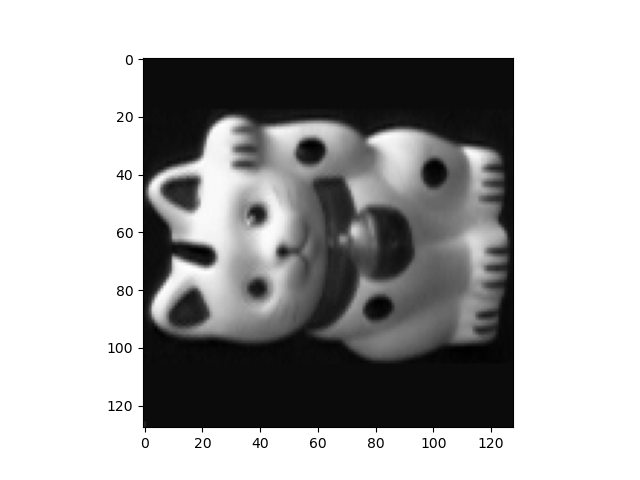}}
\subfigure[$\theta = 30 \degree, \phi = 90 \degree$]{\includegraphics[width=0.3\linewidth]{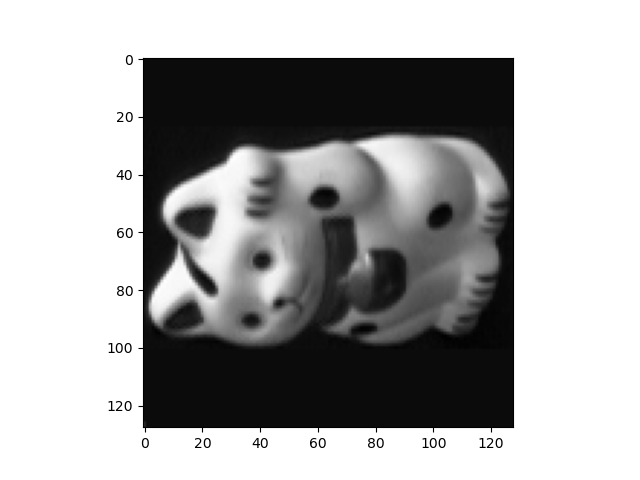}} 
\subfigure[$\theta = 120 \degree, \phi =90 \degree$]{\includegraphics[width=0.3\linewidth]{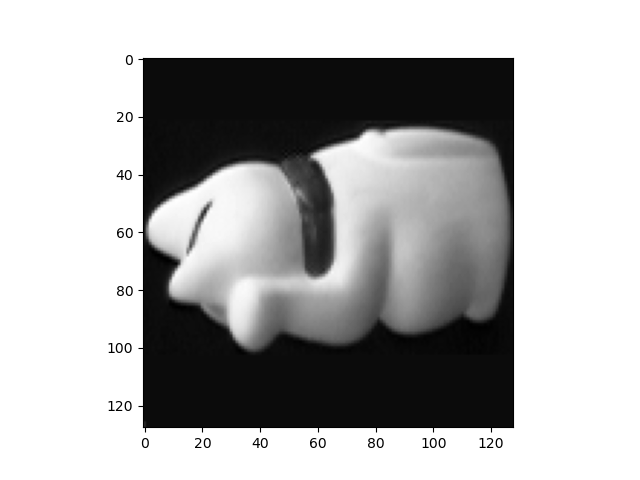}} \\

\caption{Dataset of images on a 2-dimensional manifold.}
\label{2D_dataset}
\end{center}
\end{figure*}

\begin{table*}
\begin{center}
\caption{RMSE of the prediction, regression on a 2D manifold with real data, average of $100$ runs.}
\vspace*{0.3cm}

\resizebox{\textwidth}{!}{
\begin{tabular}{|c|c|c|c|c|c|c|}
	\hline
train/test & Lasso & Lasso-TV  & $\Omega_1$ & $\Omega_2$ & KRR & k-NN  \\

	\hline	
    900/100  & 0.346 $\pm$ 0.031 & 0.344 $\pm$ 0.036 & 0.302 $\pm$ 0.032 & 0.291 $\pm$ 0.029 & \bf{ 0.281 } $\pm$ 0.027 & 0.312 $\pm$ 0.035 \\
	\hline
	800/200  & 0.351 $\pm$ 0.033 & 0.317 $\pm$ 0.030 & 0.283 $\pm$ 0.028 & \bf{0.273} $\pm$ 0.029 & 0.290  $\pm$ 0.025 & 0.309 $\pm$ 0.024 \\
	\hline
	700/300  & 0.348 $\pm$ 0.028 & 0.314 $\pm$ 0.029 & 0.293 $\pm$ 0.025 & \bf{0.280} $\pm$ 0.024 & 0.296 $\pm$ 0.023 & 0.320 $\pm$ 0.023 \\
	\hline
	600/400 & 0.344 $\pm$ 0.027 & 0.306 $\pm$ 0.025 & 0.304 $\pm$ 0.028 & 0.297 $\pm$ 0.027 & \bf{0.290} $\pm$ 0.017 & 0.333 $\pm$ 0.024 \\
\hline
\end{tabular}}
\label{table_cat2D}
\end{center}
\end{table*}

Among all prediction methods that make use of a Laplacian eigenbasis decomposition, $\Omega_2$ on a subset of eigenvectors preselected thanks to $\Omega_1$ offers the best performance when predicting to new data. Our method is overall comparable to Kernel Ridge Regression.

\subsubsection{Electrical consumption dataset}

We have tried our method on the electrical consumption dataset.\footnote{\url{https://www.kaggle.com/nicholasjhana/energy-consumption-generation-prices-and-weather}} The covariates are curves of temperature in Spain, averaged over a week. There is a measurement for each hour of the day, thus each curve can be seen as a point of $\mathbb{R}^{24}$. However, the possible profiles of temperature are very limited (namely, each curve is increasing towards a maximum in the afternoon and is then decreasing), it is therefore believed that the data lie on a manifold of smaller dimension. For each week, we try to predict the electrical consumption in Spain in GW. Like in the previous experiments, the data are randomly split between a training and a testing set, and the graph Laplacian is built on all the available data. The results can be found Table \ref{table:elec}.

\begin{table*}
\begin{center}
\caption{RMSE of the prediction of the average electrical consumption.}

\resizebox{\textwidth}{!}{
\begin{tabular}{|c|c|c|c|c|c|c|}

	\hline
  train/test & Lasso & Lasso-TV  & $\Omega_1$ &$\Omega_2$ & KRR & $k-$NN \\

	\hline	
	 108/100 & 1.289 $\pm$ 0.067  & 1.216 $\pm$ 0.065 & 1.174 $\pm$ 0.072 & 1.251 $\pm$ 0.063 & \bf{1.168} $\pm$ 0.072 & 1.188 $\pm$ 0.064 \\
	\hline
	 58/150  & 1.254 $\pm$ 0.044 & 1.265 $\pm$ 0.050 & 1.181 $\pm$ 0.048 & \bf{1.165} $\pm$ 0.046 & 1.193 $\pm$ 0.051 & 1.211 $\pm$ 0.059 \\
	 \hline
\end{tabular}}
\label{table:elec}
\end{center}
\end{table*}

On this dataset, we notice once again that a small training set is not really damaging towards topological methods as opposed to more standard methods, illustrating once again the generalization properties of the penalties $\Omega_1$ and $\Omega_2$. Here, all methods act relatively similarly as they all predict well the electrical consumption over a week when it takes values around its mean (see Figure \ref{results_elec}). However, sometimes the electrical consumption over a week reaches a peak and all methods fail to capture the extreme values of the source function. It makes perfect sense here that the knowledge of the temperature alone fails to explain perfectly the global electrical consumption over an entire country, without taking into account any other covariates.

\begin{figure}[t]
\begin{center}
\includegraphics[width=1\linewidth]{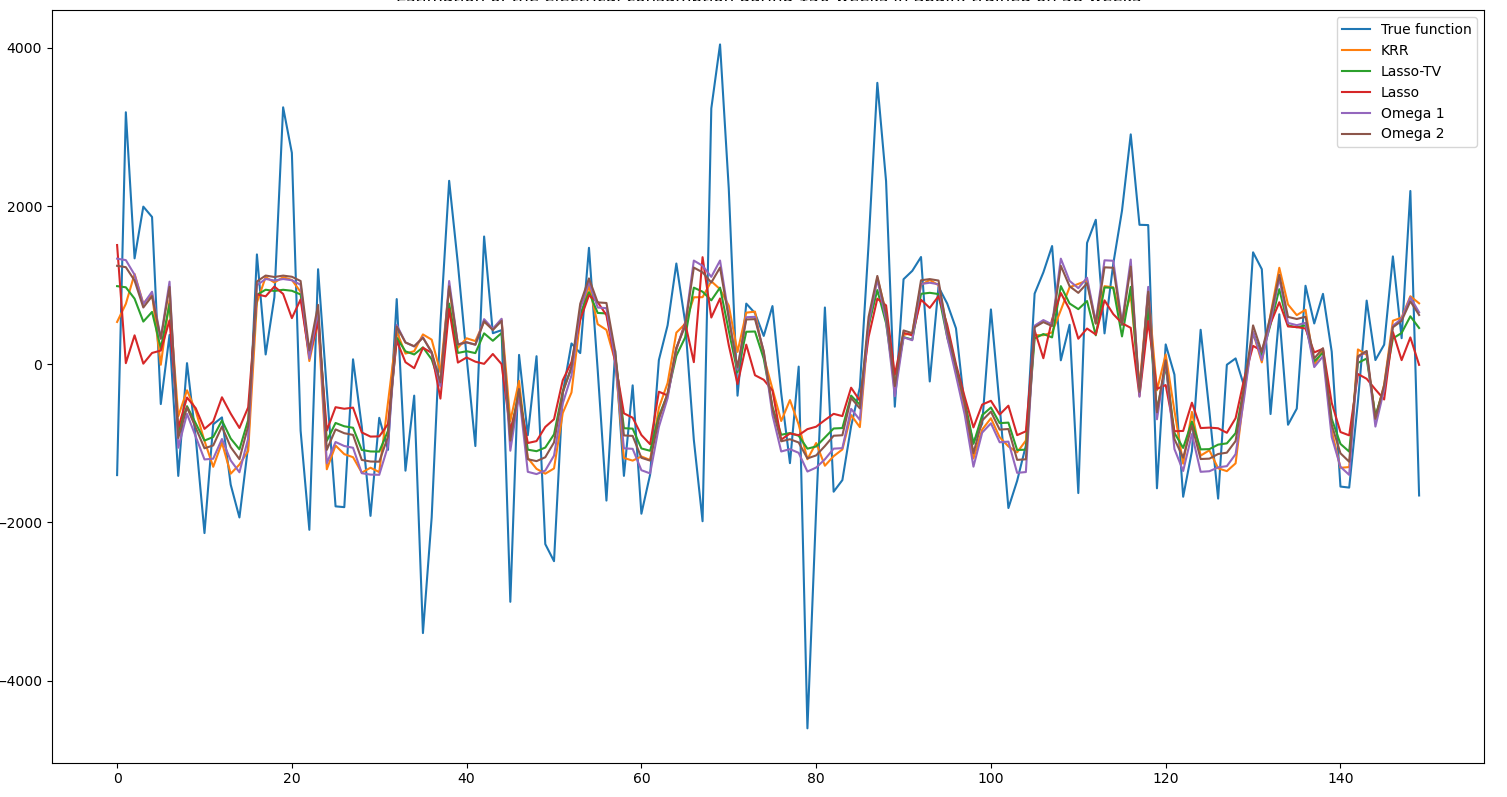}
\end{center}
\caption{Prediction of the electrical consumption using various methods.}
\label{results_elec}
\end{figure}

\subsection{Discussion on the computational cost}
\label{complexity}

The good performance of topological penalties, both in terms of prediction and reconstruction have to be nuanced by their computational cost. Table \ref{table_cost1} shows the computational time in seconds on a standard laptop without GPU for the example of Section \ref{sec:torus} (radial peak function on a torus). All the methods have been implemented in Python using standard libraries, except for $\Omega_2$ and $TV$ for which the optimization of the loss function has been implemented from scratch. This explains a higher cost as opposed to standard regression methods that already benefit from an optimized implementation. This table presents the cost for various methods for single-choices of parameters and without any cross validation. We observe that computing the graph Laplacian and its spectrum is very fast (less than a second if we have $1000$ points and ask for the entire spectrum), and performing a Lasso has a negligible cost. Most of the time is actually spent computing persistence diagrams : for the penalty $\Omega_1$, we have to compute the persistence of each eigenfunction prior to performing a standard Lasso. This turns out to be very costly if we ask for the whole spectrum and somehow reasonable if we ask for a small value of $p$. Note that once these persistences have been computed, estimating a new signal at the same data points can be done almost instantly. Penalty $\Omega_2$ has a very high computational cost, which does not decrease significantly with the dimension $p$. This is due to the fact that the persistence diagram of the entire function has to be computed at each gradient step.

Computing the $k-$persistence of a simplicial complex with $N$ simplices is done with the Gudhi package \citep{maria2014gudhi} which relies on the algorithm of \citet{edelsbrunner2010computational} and has an algorithmic complexity  of $O(N^3)$. Note that if we have $n$ data-points, the number of $k-$dimensional simplices is of order $n^k$. This accounts for the very high computational cost of topological penalties developed in this paper. It is therefore infeasible to use this method in practice for high homological dimensions, namely as soon as $k \geq 4$; in practice we recommend to only penalize $k$-persistences, for $k$ up to 3. On the other hand, note that the 
computation time is barely impacted by the intrinsic dimension of the manifold nor the ambient dimension. We remark that when minimizing $\Omega_2$, the function does not change much from one epoch to another and therefore, neither does its persistence diagram. Recomputing the diagram at each iteration is therefore a naive approach and this is a possible way of improving the method on the numerical side. We believe that using vineyards \citep{cohen2006vines} would enable a computation of the persistence diagram in linear time which would speed up the method, maybe at the cost of a higher space complexity.

\begin{table*}
\begin{center}
\caption{Computational time (in seconds), data on a torus}
\vspace*{0.3cm}

\begin{tabular}{|c|c|c|c|c|c|c|c|}
	\hline
$n$ &$p$ & Lasso & $\Omega_1$ & $\Omega_2$ & KRR & $k$-NN & TV  \\
\hline
300 & 100 & 0.04 & 4.0 & 36.3 & 0.11 & $3.0 \times 10^{-3}$ & 2.3 \\
	\hline	
1000 & 100 & 0.19 & 14.2 & 129.9 & 0.13 & $5.6 \times 10^{-3}$ & 3.6 \\

	\hline	
300 & 300 & 0.11 & 8.5 & 39.7 & 0.11 & $3.0 \times 10^{-3}$ & 2.3 \\
	\hline	
1000 & 1000 & 0.94 & 132.3 & 195.9 & 0.13 & $5.6\times 10^{-3}$ & 3.6 \\
	\hline	

\end{tabular}
\label{table_cost1}
\end{center}
\end{table*}

\subsection{Conclusion of the experiments}

The methods developed in this paper couple the use of a graph Laplacian eigenbasis and a topological penalty. The first aspect enables to treat regression problems for data living on manifolds with an extrinsic approach: nothing needs to be known on the manifold and its metric, the graph Laplacian being computed on the ambient metric. Laplace eigenmaps methods in general have proven to be useful when the manifold structure is quite strong, illustrated here with the experiment on a Swiss roll. The use of a topological penalty has multiple advantages: it acts as a generalization of total variation penalties in higher dimensions and seems to be a more natural way to regularize functions, as observed in Section \ref{sec:motiv}. Numerically, topological methods almost always provide better results than usual penalties such as TV or $L^1$ penalty. Here, we have developed two types of topological penalties: $\Omega_1$ aims at performing a selection process of the regression basis functions, by discarding the ones that oscillate too much in order to allow a good generalization to new data. On the other hand $ \Omega_2$ directly acts on the topology of the source function and performs a strong denoising. Note that the statistical noise on the data translates into a topological noise on the persistence diagram of the corresponding function which is the one the regularization $\Omega_2$ acts onto, by providing a powerful simplification of the persistence diagram of the reconstructed function. When the underlying geometric structure is complex or the noise is important, $\Omega_2$ regularization on the eigenfunctions selected by $\Omega_1$ provides a better reconstruction in terms of RMSE than standard methods, including KRR which appears to be the most competitive one.

We have essentially penalized the total persistence of functions, but it would be possible to numerically penalize any smooth function on the points of the persistence diagram. In particular, establishing a penalty that would erase all low-dimensional persistence features while keeping all the high-dimensional features could be of practical interest, likewise to the smoothly clipped absolute deviation (SCAD) penalty developed by \cite{fan2001variable}.

One major drawback of topological methods is their computational cost as opposed to a simple KRR regression, especially when we need to compute topological persistence of high dimensions.

\section{Proofs for Section~\ref{sec:theo}}
\label{sec:proofs}

\subsection{Proof of Theorem \ref{oracle_theo}}

We start by the proof of Theorem \ref{oracle_theo} since it introduces a number of lemmas and results that will also be useful to prove the other theorems. We start with a  technical lemma based on the celebrated Weyl's Law~\citep{chavel1984eigenvalues, ivrii2016100}, reformulated for statistical regression purposes in \cite{barron1999risk}. It provides a control of the sup-norm of the sum of the squares of the first Laplacian eigenfunctions.

\begin{lemma}
\label{sup_norm}
Let $\mathcal{M}$ be a compact Riemannian manifold of dimension $d$ and $\Delta$ its Laplace-Beltrami operator. Let $\lambda_1 \leq \lambda_2 \leq \ldots \leq \lambda_p \leq \ldots$ be the eigenvalues of $\Delta$, and for an eigenvalue $\lambda_i$, we denote by $\Phi_i$ a normalized eigenfunction. Then there exists constants $C^{\prime}(\mathcal{M})$ and $C(\mathcal{M})$ depending only on the geometry of $\mathcal{M}$ (with $C(\mathcal{M})$ also potentially depending on the dimension) such that for all $p$,
\[\qquad \qquad \left\| \sum_{i=1}^p\Phi_i^2 \right\|_\infty \leq C^{\prime}(\mathcal{M}) \lambda_p^{d/2} \leq C(\mathcal{M}) p,
\]
where the last inequality follows by Weyl's law.
\end{lemma}
We will also need the following concentration result, itself using Lemma \ref{sup_norm}. Recall that a random variable $\epsilon$ is called sub-Gaussian with parameter $\sigma^2 > 0$ if ${\mathbb{E}}\exp(\lambda \epsilon) \le \exp(\sigma^2\lambda^2/2)$ for all $\lambda \in {\mathbb R}$.
%
%
\begin{lemma}
\label{lemma:concentration1}
Denote for every $X \in \mathcal{M}$, the tuple $\Phi(X) = (\Phi_1 (X), \ldots, \Phi_p(X))$. Let $\varepsilon_1\ldots,\varepsilon_n$ be i.i.d sub-Gaussian random variables with parameter $\sigma^2$, independent of $X_1,\ldots,X_n$, and let $x > 0$. Then, with probability larger than $1 - 2e^{-x}$,

\[ \left\| \frac{1}{n} \sum_{i=1}^n \varepsilon_i \Phi( X_i) \right\| \leq 2 \sigma \sqrt{\frac{p}{n}} (1+2 \sqrt{x}) \sqrt{1+C(\mathcal{M})\ \sqrt{\frac{2x}{n}}}.
\]
\end{lemma}



\begin{proof}{(of Lemma \ref{lemma:concentration1}).}
Our proof is based on a non-standard result by \cite{kontorovich2014concentration}, on concentration of Lipschitz functions of not necessarily bounded random vectors, which we briefly introduce next. 

For $i=1, \ldots, n$, let $Z_i$ be random objects living in a measurable space $\mathcal{M}_i$ with metric $\rho_i$, endowed with measure $\mu_i$. Define $Z=(Z_1,\ldots, Z_n)$ to be living on the product probability space $\mathcal{M}_1 \times \cdots \times \mathcal{M}_n$ with product measure $\mu_1 \times \cdots \times \mu_n$, and metric $\rho = \sum_{i=1}^n \rho_i$. To each $(Z_i, \mu_i, \rho_i)$, we also associate symmetrized random objects $\Xi_i = \gamma_i \rho_i(Z_i,Z'_i)$, where $Z_i,Z'_i \sim \mu_i$ are independent, and $\gamma_i$ are Rademacher random variables (i.e. taking values $\pm 1$ with probability 1/2) independent of $Z_i, Z'_i$. We also define  $\Delta_{\mathrm{SG}}(\Xi_i)$ as the sub-Gaussian diameter of $\Xi_i$, given by the smallest value of $s$ for which we have $\mathbb{E} e^{\lambda \Xi_i} \leq  e^{s^2 \lambda^2/2}$ for all $\lambda \in \mathbb{R}$. Then using \cite[Proof of Theorem 1]{kontorovich2014concentration}, we have for the random object $Z$ as above, and a $1$-Lipschitz function $\varphi:\mathcal{M}_1 \times \cdots \times \mathcal{M}_n \to \mathbb{R}$, that
\begin{align}\label{probbound}
\mathbb{P}(\varphi(Z)-\mathbb{E} \varphi(Z) >t) \leq \exp \left(-\frac{t^{2}}{2 \sum_{i=1}^n \Delta_{\mathrm{SG}}^{2}\left(Z_i\right)}\right).
\end{align}

In our context, for all $1\leq i\leq n$, we set $\mathcal{M}_i =\mathbb{R}^p$, and $\rho_i= \| \cdot\|$. For $v_i\in\mathbb{R}^p$, we define the function $\varphi$ as 
$$ \varphi : (v_1, \ldots, v_n) \mapsto \left\| \sum_{i=1}^n v_i \right\|.$$ Note that, for two finite sets of vectors $(v_i)_i$ and $(w_i)_i$, we have $$\Big| \big\|\sum v_i\big\| - \big\|\sum w_i\big\|\Big| \le \sum\|v_i - w_i\|.$$ Hence, the function $\varphi$ is $1-$Lipschitz with respect to the mixed $\ell^1/\ell^2$ norm of $(\mathbb{R}^p)^n$, meaning that we will be able to apply Kontorovich's result.

Now, considering $\left\| \sum_{i=1}^n \varepsilon_i \Phi (X_i) \right\|,$ we proceed by conditioning on $(X_i)_{1\leq i \leq n}$, that is, we consider randomness only with respect to $(\varepsilon_i)_{1\leq i\leq n}$,
which are independent of the $X_i$'s;  $\mathbb{E}_\varepsilon$ will denote the corresponding conditional expectation. First, note that we have the following bound on the (conditional) expectation:
\begin{align}\label{expecbound}
\mathbb{E}_{\varepsilon} \left( \left\| \sum_{i=1}^n \varepsilon_i \Phi (X_i) \right\| \right) \leq
\mathbb{E}_{\varepsilon} \left( \left\| \sum_{i=1}^n \varepsilon_i \Phi (X_i) \right\|^2 \right)^{\frac{1}{2}} \le 
\sqrt{\sigma^2 \sum_{i=1}^n \|\Phi(X_i)\|^2}.
\end{align}
Defining $Z_i = \varepsilon_i\Phi(X_i)$, and considering the corresponding symmetrized object 
$$\Xi_i = \gamma_i \left\| \varepsilon_i \Phi(X_i) - \varepsilon_i' \Phi(X_i)\right\|
= \gamma_i | \varepsilon_i - \varepsilon_i'| \| \Phi(X_i) \|,$$ 
and since $\gamma_i | \varepsilon_i - \varepsilon_i'|$ has the same distribution as $(\varepsilon_i - \varepsilon_i')$ by independence and symmetry, we have
$$
\mathbb{E}_\varepsilon \left( \exp(\lambda \Xi_i) \right) = \mathbb{E}_\varepsilon
\left( \exp ( \lambda (\varepsilon_i - \varepsilon_i') \| \Phi(X_i) \| )\right)
\leq \exp( \lambda^2 \sigma^2 \| \Phi(X_i) \|^2)
$$
since the $\varepsilon_i,\varepsilon'_i$ are independent sub-Gaussian of parameter $\sigma^2$, hence we also have for the (conditional) sub-Gaussian diameter $\Delta^2_{\mathrm{SG}}\left(\Xi_i\right) \leq  2\sigma^2 \|\Phi(X_i)\|^2$. Hence, by using~\eqref{probbound} and~\eqref{expecbound}, we have that with probability larger than $1-e^{-x}$,
\begin{align*}
    \left\|\frac{1}{n}\sum_{i=1}^n \varepsilon_i \Phi(X_i) \right\| \leq 2\frac{\sqrt{\sigma^2 \sum_{i=1}^n \|\Phi(X_i)\|^2 x}}{n}    +\frac{\sqrt{\sigma^2 \sum_{i=1}^n \|\Phi(X_i)\|^2}}{n}.
\end{align*}

We now deal with removing the conditioning on $X_i$. Note that according to Lemma \ref{sup_norm}, we have for all $i$: $\|\Phi(X_i)\|^2 = \sum_{j=1}^p \Phi_j(X_i)^2 \leq C(\mathcal{M}) p$. We can therefore apply Hoeffding's inequality and obtain that with probability larger than $1-e^{-x}$,
\[
\qquad \qquad \sum_{i=1}^n \| \Phi(X_i) \|^2 \leq np+C(\mathcal{M})p \sqrt{2nx}  
= np\Big( 1+ C(\mathcal{M}) \sqrt{\frac{2x}{n}}\Big),
\]
where we used that ${\mathbb E}\|\Phi(X_i)\|^2 = np$, because the eigenfunctions are normalized. We therefore obtain the desired result. 
\end{proof}

For a given function $f$ on a probability space ($\mathcal{X}, \pi)$, we define the expectation operator $P(f) = \int f \mathrm{d}\pi$. Given $n$ i.i.d. random variables on $\mathcal{X}$, we define the empirical measure $P_n (f) = \frac{1}{n} \sum_{i=1}^n f(X_i)$. The following lemma provides a control of the empirical process of the difference between the true function and its estimation:

\begin{lemma}
\label{lemma:concentration2}
With the same notation as before, 
we have that with probability larger than $1-\exp \left(\frac{-0.1\,n}{C(\mathcal{M})p} + \ln (2p) \right)$:
\[
 \sup_{\|\beta\|=1} (P_n-P)(\langle \beta, \Phi (X) \rangle^2) \leq \frac{1}{2}.
\]
\end{lemma}

\begin{proof}{(of Lemma \ref{lemma:concentration2}).}
First note that for the empirical process we have

\begin{align*}
\sup_{\|\beta\|=1} (P_n-P)(\langle \beta, \Phi (X) \rangle^2) &=  
\sup_{\|\beta\|=1} \beta^t \left( \frac{1}{n} \sum_{i=1}^n \Phi(X_i)\Phi(X_i)^T - \mathbb{E}[\Phi (X) \Phi(X)^T]\right)\beta\\
& = \left\| \frac{1}{n} \sum_{i=1}^n \Phi(X_i)\Phi(X_i)^T - I_p \right\|_{\mathrm{op}},
\end{align*}
where we have used $ \mathbb{E}[\Phi (X) \Phi(X)^T] = I_p$, since the components
$(\Phi_i)_{1\leq i \leq p}$ of $\Phi$ form an orthonormal system of $L^2 (\mathcal{M})$.
Note that  $\|\Phi (X_i) \Phi(X_i)^T\|_{\mathrm{op}} = \|\Phi(X_i)\|^2$ is upper-bounded by $L=C(\mathcal{M})p$ according to Lemma \ref{sup_norm}. We then use Theorem 5.1.1 from \cite{tropp2015introduction} which yields that: 
\begin{equation} \label{eq:concmat} \mathbb{P} \left(
\left\| \frac{1}{n} \sum_{i=1}^n \Phi(X_i)\Phi(X_i)^T  - I_p \right\|_{\mathrm{op}} \geq \frac{1}{2}\right)
\leq 2p K^{\frac{n}{C(\mathcal{M})p}}.
\end{equation}
A simple calculation and evaluation of the numerical constant $K=\max \Big(\frac{e^{-1/2}}{(1/2)^{1/2}},\frac{e^{1/2}}{(3/2)^{3/2}} \Big)$ gives the required result.
\end{proof}

We are now in a position to prove Theorem \ref{oracle_theo}. For $f_{\theta}=\sum_{i=1}^p \theta_i \Phi_i$, let the quadratic loss be denoted by %
\[
\gamma(\theta, (x,y)) = (f_\theta(x) - y)^2 = ( \langle \theta, \Phi(x) \rangle - y)^2.
\]
We first remark that since the $\Phi_i$'s are an orthonormal system,  $ \mathbb{E} (f_\theta (X) - f_{\theta^\star} (X))^2  = \| \theta - \theta^\star \|^2$. 
Moreover, it is a well-known property of the quadratic loss that the excess risk 
of any prediction function $f$ is the squared $L^2$ distance to the optimal regression function
$f^*=f_{\theta^*}$, so
that $P(\gamma(\theta, (X,Y))) - P(\gamma(\theta^\star, (X,Y)))
= \mathbb{E}(f_\theta (X) - f_{\theta^\star} (X))^2 =  \|\theta - \theta^\star \|^2$.
Since the test data point $(X,Y)\sim P$ used to compute the risk is independent of the sample used to construct the estimator $\hat\theta$, we also have $P(\gamma(\hat\theta, (X,Y))) - P(\gamma(\theta^\star, (X,Y))) = \|\hat \theta - \theta^*\|^2$ (conditionally on the training sample).

Since $\hat{\theta}$ is a minimizer of $\mathcal{L}$ given in~\eqref{eq:penloss}
with $\Omega_2(\theta) = \chi(f_\theta)$, we have that
\begin{equation}
    \label{eq:lopt}
P_n (\gamma (\hat{\theta},.)) - P_n (\gamma (\theta^\star,.)) + \mu \chi (f_{\hat{\theta}}) - \mu \chi (f_{\theta^\star})
\leq 0.
\end{equation}
Therefore, we have
\begin{align*}
\|\hat{\theta} - \theta^*\|^2  &= P {\gamma} (\hat{\theta}, \cdot) - P {\gamma} (\theta^\star, \cdot)\\ 
&\leq (P-P_n)({\gamma} (\hat{\theta}, \cdot) -  {\gamma} (\theta^\star, \cdot)) + \mu (\chi (f_{\theta^\star}) - \chi (f_{\hat{\theta}})).\\
&\leq (P-P_n) (-2(\langle \theta^\star, \Phi \rangle + \varepsilon)\langle \hat{\theta}, \Phi \rangle + \langle \hat{\theta}, \Phi \rangle^2 + 2(\langle \theta^\star, \Phi \rangle + \varepsilon) \langle \theta^\star, \Phi \rangle - \langle \theta^\star, \Phi \rangle^2 ) \\
 & \qquad \qquad + \mu (\chi (f_{\theta^\star}) - \chi (f_{\hat{\theta}})) \\
 &\leq \underbrace{(P_n-P) (2 \varepsilon \langle \hat{\theta} - \theta^\star, \Phi \rangle)}_{A}+ \underbrace{(P_n-P)(\langle \hat{\theta} - \theta^\star, \Phi \rangle^2)}_{B} + \underbrace{\mu (\chi (f^\star) - \chi (\hat{f}))}_{C}.
\end{align*}
We will now bound terms A, B and C  below. First note that
\begin{align*}
A \leq 2\| \theta^\star - \hat{\theta} \| \underset{\theta}{\sup} \left\langle \frac{\theta - \theta^\star}{\|\theta - \theta^\star\|}, \frac{1}{n}\sum_{i=1}^n \varepsilon_i \Phi (X_i) \right\rangle &\leq 2\|\hat{\theta} - \theta^\star \| \left\| \frac{1}{n} \sum_{i=1}^n \varepsilon_i \Phi (X_i) \right\|,
\end{align*}
and that Lemma \ref{lemma:concentration1} provides control of the norm on the right-hand side. Next note that according to Lemma \ref{lemma:concentration2}, we have 
$$B \le \frac{1}{2} \|\hat \theta - \theta^*\|.$$
Finally, to control term $C$, we simply apply Lemma \ref{stability}, which yields
\begin{equation*}
\begin{split}
C = \mu (\chi (f^\star) - \chi (\hat{f})) & \leq \mu (2\nu (f^\star)+\zeta) \|\hat{f}-f^\star \|_\infty \\
& = \mu (2\nu (f^\star)+\zeta) \| \langle \hat{\theta} - \theta^*, \Phi(\cdot) \rangle \|_\infty \\
& \leq \mu (2\nu (f^\star)+\zeta) \|\hat{\theta} - \theta^\star\| \left\|\sum_{i=1}^p \Phi_i^2 \right\|^{1/2}_\infty \\
 & \leq \mu \sqrt{C(\mathcal{M})}(2\nu (f^\star)+\zeta) \| \hat{\theta} - \theta^\star \| \sqrt{p},
\end{split}
\end{equation*}
where the second to last inequality uses the Cauchy-Schwarz inequality, and the last is using Lemma \ref{sup_norm}. Finally, we have, with probability larger than $1-2e^{-x}-\exp \left(\frac{-0.1n}{C(\mathcal{M})p} + \ln (2p) \right)$,

\begin{align*}
 \| \hat{\theta}-\theta^\star \|^2 \leq &~2\| \hat{\theta}-\theta^\star \|\; 2 \sigma\sqrt{\frac{p}{n}} (1+ \sqrt{x}) \sqrt{1+C(\mathcal{M})\sqrt{\frac{2x}{n}}}  \\
 &~+ \frac{1}{2} \|\hat{\theta} - \theta^\star \|^2 + \| \hat{\theta} - \theta^\star \| \mu 
 \sqrt{C(\mathcal{M})}(2\nu (f^\star)+\zeta) \sqrt{p}.
\end{align*}
A simple calculation then gives the first claim.

To prove the second claim of Theorem \ref{oracle_theo}, notice that 
from~\eqref{eq:lopt} we deduce that
\begin{align*}
\chi(f_{\hat{\theta}}) \leq \chi (f_{\theta^\star}) + \frac{1}{\mu} (P-P_n) ({\gamma} (\hat{\theta}, \cdot) - {\gamma} (\theta^\star, \cdot)) + \frac{1}{\mu} (P {\gamma} (\theta^\star, \cdot) -P {\gamma} (\hat{\theta}, \cdot)).
\end{align*}
Since $\frac{1}{\mu} (P {\gamma} (\theta^\star, \cdot) -P {\gamma} (\hat{\theta}, \cdot)) \leq 0$, the claim immediately follows from the same arguments as in the first part (control of terms A and B, and
reinjecting the control for $\|\hat{\theta}-\theta^\star \|$).

\subsection{Proof of Theorem \ref{theo:lasso}}

We want to use the results of \citet[Section 6]{buhlmann2011statistics} on the Lasso.
The design matrix $\mathbf{X}$ here is given by $\mathbf{X}_{i,j} = \Phi_j (X_i)$,
i.e. the $i$-th column of $\mathbf{X}$ is $\Phi(X_i)$, denoting as done earlier
$\Phi(x) = (\Phi_1(x),\ldots,\Phi_p(x)) \in \mathbb{R}^p$.
The empirical Gram matrix is
\[
\hat{\Sigma} = \frac{1}{n} \mathbf{X}^T\mathbf{X} = \frac{1}{n} \sum_{i=1}^n \Phi(X_i) \Phi(X_i)^T.
\]
Following the same argument as in the proof of Lemma~\ref{lemma:concentration2} leading up to~\eqref{eq:concmat}, we have $\mathbb{E}(\hat \Sigma) = I_p$ and, according to Theorem 5.1.1 of \cite{tropp2015introduction}: 
\[ \mathbb{P} (\Lambda_{\min} (\hat{\Sigma}) \leq 1/2) \leq p \left( \frac{e^{-1/2}}{\sqrt{1/2}} \right)^{\frac{n}{C(\mathcal{M})p}}.
\]
(The minor difference in comparison to~\eqref{eq:concmat} is due to the fact that we only 
need the upper bound on the largest eigenvalue from Theorem 5.1.1 of \citealp{tropp2015introduction} here.)
Therefore, we have
\[
\mathbb{P} (\Lambda_{\min} (\hat{\Sigma}) \geq 1/2) \geq 1-p\exp \left(\frac{n}{C(\mathcal{M})p} \ln \left( \frac{2e^{-1/2}}{\sqrt{2}} \right) \right) \ge 1 - \exp\left( -\frac{0.15n}{C(\mathcal{M})p}+\ln(p)\right)
\]
This shows that, for $p\log p/n$ large, the smallest eigenvalue of the empirical Gram matrix is larger than $1/2$ with high probability, and the compatibility condition is verified with a constant larger than $1/2$.

We can now use the results of \citet[Section 6]{buhlmann2011statistics}. Their Theorem 6.1 remains true for the norm $I(\theta) = \sum_{i=1}^p |\theta_i| \chi (\Phi_i)$ and a compatibility constant larger than $1/2$ with large probability. For a given $\mu_0$, we therefore have, for $\mu \geq 2\mu_0$, with probability larger than $1-\exp\left(-\frac{0.15n}{C(\mathcal{M})p}+\ln(p)\right)$,
\[
\frac{1}{n} \|\mathbf{X}(\hat{\theta}-\theta^\star)\|_2^2  + \mu \sum_{i=1}^p \chi(\Phi_i) |\hat{\theta}_i - \theta_i^\star| \leq \frac{\mu^2 s}{2},
\]
on the event
\[
\mathcal{E}:=\left\{\max _{1 \leq j \leq p} \frac{2}{n} \left|\sum_{i=1}^n \varepsilon_i \Phi_j (X_i) \right|  \leq \mu_{0}\right\}.
\]
For a well chosen value $\mu_0$, the event $\mathcal{E}$ is realized with large probability. Indeed, sub-Gaussianity of the $\varepsilon_i$'s (with parameter $\sigma^2$) yields, by using similar arguments as given in the proof of Theorem~\ref{oracle_theo}, that conditionally on the $X_i$'s, the random variable $\sum_{i=1}^n \varepsilon_i \Phi_j(X_i)$ is sub-Gaussian with parameter $\sigma^2 \sum_{i=1}^n\Phi_j(X_i)^2$. We thus obtain
\[\mathbb{P}_{\varepsilon} \left(\frac{1}{n}\left|\sum_{i=1}^n \varepsilon_i \Phi_j(X_i) \right| > \mu_0 \right) \leq2\exp \left( \frac{- n^2\mu_0^2}{4 \sigma^2 \sum_{i=1}^n \Phi_j (X_i)^2} \right),
\]
where we use the fact that for any sub-Gaussian random variable $Y$ with parameter $\tau^2$, we have $P(|Y| > \lambda) \le 2 \exp(-\lambda^2/4\tau^2)$ (e.g. see \citet{vershynin2018high}. Furthermore, Lemma \ref{sup_norm} gives that
$\Phi_j (X_i)^2 \leq \sum_{j=1}^p \Phi_j(X_i)^2 \leq C(\mathcal{M}) p$ almost surely, and thus we have
\[\mathbb{P}_{\varepsilon} \left(\frac{1}{n}\left|\sum_{i=1}^n \varepsilon_i \Phi_j(X_i) \right| > \mu_0 \right) \leq 2\exp \left( \frac{- n\mu_0^2}{4 \sigma^2 C(\mathcal{M}) p}\right).
\]
A simple union-bound and the choice  $$\mu_0 = 2 \sigma \sqrt{\frac{pC(\mathcal{M})(\ln(p)+x)}{n}}$$ directly gives the required result.

\subsection*{Acknowledgments}
We would like to thank the Action Editor, Sayan Mukherjee, and the two anonymous reviewers for their insightful comments that helped greatly improve the exposition. We gratefully acknowledge support for this project from the National Science Foundation via grant NSF-DMS-2053918, and from the Agence Nationale de la Recherche via grant ANR-19-CHIA-0021-01. We would like to thank Mathieu Carrière for his implementation of the stochastic gradient descent for persistence-based functionals.




%



\bibliography{Bibliographie_Davis}

\end{document}